\documentclass[runningheads]{llncs}

\usepackage[T1]{fontenc}
\usepackage[table,dvipsnames]{xcolor}
\usepackage{amssymb}
\usepackage{cite}
\usepackage{hyperref}
\usepackage{adjustbox}
\usepackage{graphicx}
\usepackage{printlen}
\usepackage{multirow}
\usepackage{booktabs}
\usepackage{url}
\usepackage{amsmath}
\usepackage{xspace}
\usepackage[table]{xcolor}
\usepackage{lineno}
\usepackage[makeroom]{cancel}
\usepackage{array}
\usepackage[capitalize]{cleveref}
\usepackage{orcidlink}
\usepackage{float}
\usepackage{algorithm}
\usepackage{algpseudocode}

\makeatletter
\DeclareRobustCommand\onedot{\futurelet\@let@token\@onedot}
\def\@onedot{\ifx\@let@token.\else.\null\fi\xspace}

\def\eg{\emph{e.g}\onedot}

\makeatother

\newcommand{\perc}[1]{\textcolor{green!50!black}{\scriptsize $\blacktriangle$#1}}

\newif\ifreview
\reviewfalse

\ifreview

	\linenumbers
\fi

\begin{document}

\def\SubNumber{87}
\def\GCPRTrack{Fast Review Track}

\title{Concept Guidance: Precise, Training-Free Latent Control for Text-to-Image Generation}

\ifreview
    \titlerunning{GCPR 2026 Submission \SubNumber{}. CONFIDENTIAL REVIEW COPY.}
    \authorrunning{GCPR 2026 Submission \SubNumber{}. CONFIDENTIAL REVIEW COPY.}
    \author{GCPR 2026 - \GCPRTrack{}}
    \institute{Paper ID \SubNumber}
\else
    \titlerunning{Concept Guidance}
    \author{
        Nikolai Röhrich\inst{1,2}%
        \thanks{Corresponding author: \texttt{n.roehrich@campus.lmu.de}}%
        \textsuperscript{$\dagger$}%
        \and
        Isabell Hans\inst{1,2}%
        \textsuperscript{$\dagger$}%
        \and
        Felix Krause\inst{3,4}%
        \textsuperscript{$\dagger$}%
        {\def\lastandname{\unskip\\}%
            \and
            \mbox{Björn~Ommer}\inst{3,4}%
        }
    }
    \authorrunning{N. Röhrich et al.}
    \institute{
        LMU Munich, Germany
        \and
        Konrad Zuse School of Excellence in Reliable AI (relAI), Germany
        \and
        CompVis @ LMU Munich, Germany
        \and
        Munich Center for Machine Learning (MCML), Germany
    }
\fi

\maketitle

\ifreview\else
    \begingroup
    \renewcommand{\thefootnote}{$\dagger$}
    \footnotetext[0]{Equal contribution.}
    \endgroup
\fi

\begin{abstract}
Text-to-image diffusion models have two major drawbacks that severely limit their practical utility: (1) standard models lack an intrinsic mechanism for continuous, concept-specific guidance (\eg, for precisely controlling how \textit{aesthetically pleasing} an image looks), and (2) they lack reliability for tasks requiring high local coherence (\eg, generating \textit{text} or \textit{human hands}). To tackle these issues, we introduce a novel notion of concept-wise mutual information and find large, concept-dependent differences between individual layers, demonstrating that the generation of specific structures is localized in distinct parts of the network. We exploit this insight by reinforcing the impact of concept-relevant layers in \emph{Concept Guidance (CoG)}, a precise, target-specific guidance method that works for models out-of-the-box without additional training, external models, gradients, or prompt engineering. CoG first quantifies each layer's concept-specific impact and then guides denoising using a weighted combination of predictions generated with concept-relevant layers skipped. We demonstrate performance increases across various targets and popular models like PixArt-$\alpha$, SD3, SD3.5, and FLUX.1-dev. Code is available at \url{https://github.com/CompVis/concept\_guidance}.
\keywords{Text-to-Image Generation \and Diffusion Models \and Guidance}
\end{abstract}

\vspace{-3mm}

\section{Introduction}
\label{sec:intro}

Text-to-image (T2I) diffusion models have achieved remarkable quality in generating images from natural language descriptions \cite{nichol2021glide, rombach2022high, chen2023pixart, esser2024scaling}. By sampling a random latent and iteratively refining it, these models gradually transform the latent into a coherent image that aligns with a given text prompt. Two fundamental challenges persist, particularly in one-shot generation. First, T2I models lack \textit{reliability} in tasks requiring precise local coherence. This is perhaps most apparent in outputs involving text, where models frequently produce misspellings or hallucinated characters \cite{chen2023textdiffuser}, and in images involving human hands, which often exhibit incorrect finger counts, distortions, or implausible geometry \cite{narasimhaswamy2024handiffuser}.

\begin{figure}[t]
    \centering
    \includegraphics[width=0.76\linewidth]{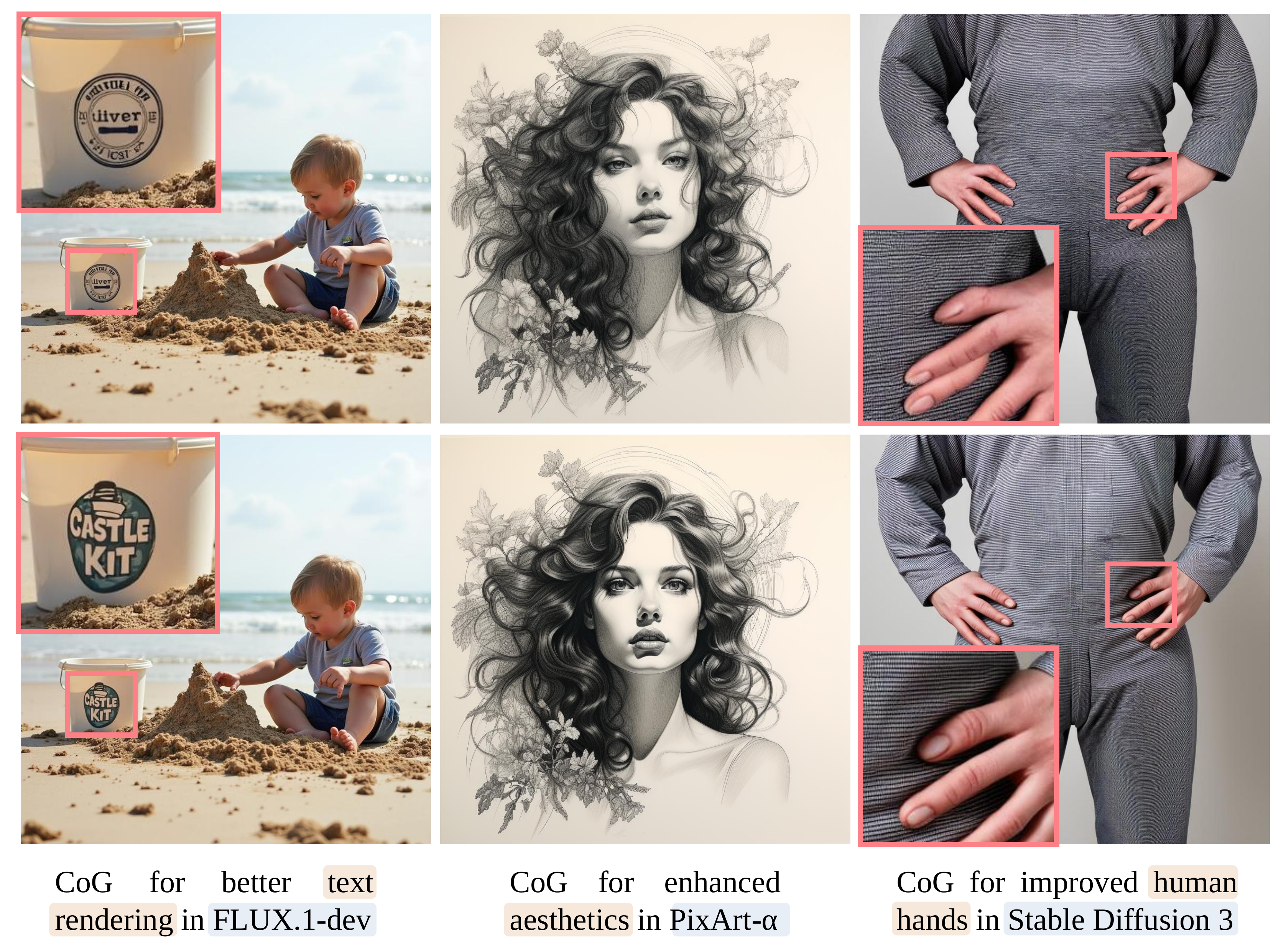}
    \caption{\textbf{Concept Guidance (CoG) improves arbitrary quantifiable concepts across a variety of T2I models.} CoG enables models to recover from classical failure modes while also delivering better generative quality for abstract concepts.}
    \label{fig:teaser}
\end{figure}

Second, Classifier-Free Guidance (CFG) \cite{ho2022classifier_free_guidance} -- the de facto standard guidance mechanism in T2I diffusion -- provides no \textit{fine-grained control} over the generation process. In CFG, generation is guided by performing a conditional and an unconditional forward pass and by extrapolating beyond the conditional noise prediction. This technique effectively controls global prompt alignment, but it entirely lacks the ability to control fine-grained semantic details and tends to underperform for complex prompts due to its global nature \cite{liu2022compositional, wu2023harnessing}.

Prominent attempts to address these limitations fall into three broad categories, each with significant drawbacks (complementary directions, such as refined guidance rules, are discussed in \Cref{sec:related_work}): (1) Fine-tuning methods like ControlNet \cite{zhang2023adding_control} or T2I-Adapters \cite{mou2024t2i_adapter} can impose spatial control but require extensive per-concept training and large datasets. (2) Gradient-based methods \cite{dhariwal2021diffusion_classifier_guidance, bansal2023universal_guidance} can guide generation towards a target but are computationally expensive. (3) Inference-time intervention, e.g. by manipulating cross-attention maps \cite{hertz2022prompt, kim2025text}, offers more flexibility, but often requires careful model-specific tuning.

Our method, \textit{Concept Guidance} (CoG), in contrast, is an out-of-the-box, concept-specific guidance mechanism for T2I models that requires no training, external models, gradients, or reverse-engineering of model internals. By introducing a novel notion of per-layer, per-concept mutual information, we demonstrate the varying degrees of influence that layers in T2I models have on the generation of specific semantic concepts (see \Cref{fig:mi_combined}), and we exploit this information by reinforcing the influence of concept-relevant layers (see \Cref{fig:manifold}).

We first propose a framework to measure the layer-wise target performance by using target-specific metrics.
Then, we generate predictions with relevant layers skipped and guide by extrapolating the standard prediction away from those predictions, using performance-based weighting. This selectively amplifies the most concept-relevant layers, allowing for precise guidance. Additionally, CoG seamlessly integrates with CFG, allowing users to effortlessly steer both global prompt alignment \emph{and} targeted semantics. We make the following contributions: 

\smallskip
\begin{enumerate}
    \item We propose a novel method of measuring the mutual information between diffusion network layers and target concepts, and thereby show that the properties of specific layers can be exploited for precise, per-concept guidance. We demonstrate that layers that work well for skip-based guidance are precisely the layers responsible for generating relevant concepts.
    \item In light of our layer analysis, we present \textit{Concept Guidance}, a concept-specific guidance method based on layer skipping. CoG is a simple yet effective method that enables precise, out-of-the-box guidance for T2I models and increases T2I generation performance for arbitrary (measurable) concepts.
    \item Through extensive experiments, we demonstrate that Concept Guidance consistently outperforms Classifier-Free Guidance and alternative guidance mechanisms based on layer-skipping. Across all tested models, we achieve an average single-target performance increase of $8.1\%$ compared to CFG. 
\end{enumerate}

\section{Related Work}
\label{sec:related_work}

\paragraph{Training-Free Guidance via Model Perturbations}
Diffusion sampling is commonly steered with \emph{classifier} guidance \cite{dhariwal2021diffusion_classifier_guidance} which leverages gradients from an external classifier to steer generation, or \emph{classifier-free} guidance (CFG) \cite{ho2022classifier_free_guidance}, which increases condition adherence by contrasting conditional and unconditional predictions.
Recent work revisits guidance rules to reduce CFG artifacts, \eg, via adaptive projection (APG) \cite{sadat2024eliminating} or manifold-constrained guidance (CFG++) \cite{chung2024cfgpp}, and to obtain CFG-like behavior \emph{without} special unconditional training
\cite{sadat2024no_training}.
Autoguidance \cite{karras2024guiding_autoguidance} instead contrasts the model with a weakened version of itself, \eg, an earlier training checkpoint, thus requiring access to training artifacts.
A complementary training-free line constructs a ``weak'' model at inference time by perturbing the generator itself, including attention-maps \cite{hong2023improving,hong2024smoothed,ahn2024self}, processed tokens \cite{rajabi2025token}, or self-guidance derived from the model's own dynamics \cite{li2024selfguidance}.
Layer skipping is another perturbation method widely exposed in modern diffusion pipelines, e.g., Stable Diffusion~3~\cite{huggingfaceSD3}, yet it remains underexplored as a controllable handle. Spatiotemporal Skip-Guidance (STG) uses a fixed single-layer skip to improve video quality~\cite{hyung2025spatiotemporal}.
Our work is closest in spirit to perturbation-based training-free guidance, but differs by making skip perturbations \emph{concept-dependent} and by using \emph{weighted combinations of multiple layers} rather than a fixed skip.

\paragraph{Concept-Specific Control Beyond Prompting}
Beyond prompt engineering, concept-specific control is often achieved by adding learnable components: low-rank adapters (LoRA) \cite{hu2022lora} enable parameter-efficient updates and support per-concept modules such as concept sliders \cite{gandikota2024concept}, while other approaches train auxiliary controllers (e.g., adapters) for new conditioning modalities \cite{mou2024t2i_adapter}.
Relatedly, some methods train lightweight predictors or readout heads on frozen diffusion features and backpropagate through them during sampling to enforce targets \cite{luo2023readout}.
A different family uses \emph{external} guidance objectives at inference time---either via gradients from arbitrary guidance functions \cite{bansal2023universal_guidance} or via energy-based losses built from off-the-shelf predictors \cite{yu2023freedom}, with newer formulations addressing proxy unreliability through trust-region style sampling \cite{huang2024trust}.
In contrast, our goal is a practical mechanism for concept-specific improvement within an off-the-shelf generator: we require no training and no external predictors during denoising.

\paragraph{Concept Localization and Interpretability}
A growing body of work probes \emph{where} and \emph{how} diffusion models represent text-conditioned concepts.
Attention-centric analyses and interventions treat cross-attention as the primary locus of word-region binding  \cite{hertz2022prompt,chefer2023attend}, and attribution methods such as DAAM derive token-to-pixel maps from cross-attention aggregation \cite{tang2022daam}.
Separately, feature-level probing reveals that semantic correspondences emerge in intermediate layers and vary strongly with depth \cite{tang2023emergent,luo2023diffusion}.
Information-theoretic perspectives quantify prompt--image dependence more robustly than raw attention, using MI-style decompositions \cite{kong2024interpretable,zawar2024diffusionpid} and applying information-theoretic objectives to improve alignment \cite{wang2025information}.
Mechanistic and causal approaches further localize attribute-relevant components for model editing \cite{basu2023localizing,basu2024mechanistic} or component attribution \cite{nguyen2024unveiling}, and recent work identifies ``vital'' layers in transformer backbones for training-free editing \cite{avrahami2024stableflow}.
Our work connects these threads by turning per-concept layer specialization into an \emph{actionable} control signal: we localize concept-relevant layers through targeted interventions and use the resulting layers to guide generation.

\section{Method}

\subsection{Preliminaries}

\paragraph{T2I Diffusion} Text-to-Image diffusion models \cite{nichol2021glide, rombach2022high, esser2024scaling} generate images aligned with a text prompt by iteratively denoising a randomly sampled latent into a noise-free image \cite{ho2020denoising_ddpm,song2020denoising}. This \textit{reverse process} is trained by adding noise to image samples from the target distribution during the \textit{forward process}. Formally, given an image $x_0$, noisy latents $x_t$ at timestep $t$ are given by

\begin{equation}
    q(x_t \mid x_0) = \mathcal{N}(x_t; \sqrt{\bar{\alpha}_t} x_0, (1-\bar{\alpha}_t)I), 
\end{equation}
where $\bar{\alpha}_t$ controls the noise level. Parameters $\theta$ are updated based on the distance between predicted noise $\epsilon_\theta(x_t, t, c)$ and actual noise $\epsilon$, thus solving 

\begin{equation}
\min_{\theta} \quad \mathbb{E}_{x_0,\epsilon,t,c} \|\epsilon - \epsilon_\theta(x_t, t, c)\|^2.
\end{equation}

\paragraph{Classifier Guidance} Classifier Guidance \cite{dhariwal2021diffusion_classifier_guidance} was introduced to steer unconditional diffusion models towards a distribution aligned with a desired mode $y$. Given a classifier $p_\phi(y \mid x_t)$, guidance is achieved by modifying the process to favor samples the classifier considers more likely. The distribution is adjusted as

\begin{equation}
    p_\theta(x_{t-1} \mid x_t, y) \propto p_\theta(x_{t-1} \mid x_t) \, p_\phi(y \mid x_t),
\end{equation}
which corresponds to adding a correction to the model’s score estimate based on the gradients of the external classifier. In practice, this results in

\begin{equation}
    \nabla \log p_\theta(x_t \mid y)
    = \nabla \log p_\theta(x_t)
    + \lambda \, \nabla \log p_\phi(y \mid x_t),
\end{equation}
where $\lambda$ is a scalar guidance strength controlling the influence of the classifier on the generation process. The classifier gradients $\nabla \log p_\phi(y \mid x_t)$ can be interpreted as the direction of strongest label-alignment in the latent space, $\vec{y}$.

\paragraph{Classifier-Free Guidance} Classifier-Free Guidance \cite{ho2022classifier_free_guidance} steers diffusion generation towards a target $y$ specified by a text prompt $c$, without requiring a classifier. Instead, the model is trained to perform both unconditional and text-conditioned generation. Then, guidance is achieved by extrapolating from the unconditional prediction $\epsilon_\theta(x_t, t, \emptyset)$ beyond the conditional prediction $\epsilon_\theta(x_t, t, c)$: 

\begin{equation}
    \tilde{\epsilon}_\theta(x_t, t, c) = (1-\lambda)\epsilon_\theta(x_t, t, \emptyset) + \lambda \epsilon_\theta(x_t, t, c), 
\label{eq:cfg}
\end{equation}
where $\lambda \geq 1$ controls the strength of conditioning. Similar to Classifier Guidance, this process can be interpreted as determining and reinforcing the direction of strongest condition alignment in the latent space $\vec{y}$.

\subsection{Concept Guidance}
\label{sec:msg}

\begin{figure}[t]
\centering
\begin{minipage}[c]{0.5\linewidth}
    \centering
    \includegraphics[width=\linewidth]{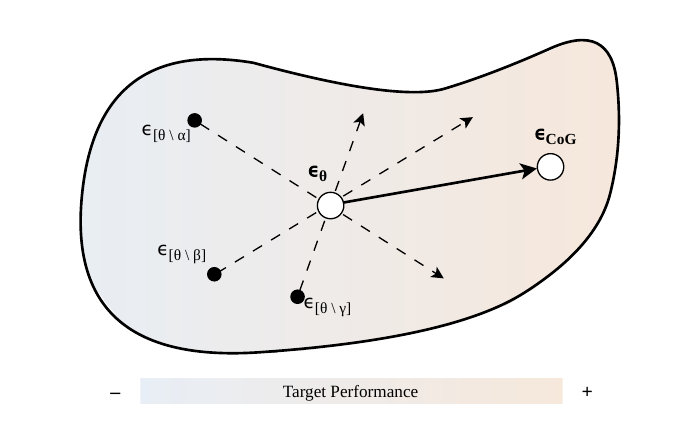}
\end{minipage}\hfill
\begin{minipage}[c]{0.44\linewidth}
    \caption{\textbf{Concept Guidance precisely approximates the target direction }by combining per-layer skip predictions using concept-relevant layers ($\alpha,\beta,\gamma$). It computes individual skip-layer noise predictions ($\epsilon_{[\theta\setminus \alpha]}, \epsilon_{[\theta\setminus \beta]}, \epsilon_{[\theta\setminus \gamma]}$) and extrapolates over them to precisely estimate the target direction.}
    \label{fig:manifold}
\end{minipage}
\end{figure}

Classifier Guidance is precise in guiding towards specific targets, while Classifier-Free Guidance is simple and effective for increasing text-to-image alignment. We introduce \textit{Concept Guidance} to combine the best of both worlds: highly usable, target-specific guidance in the latent space of diffusion models. We exploit the fact that given a target $y$, different layers of T2I models are particularly impactful w.r.t. $y$. To approximate the noise-space direction of highest target performance $\vec{y}$, we find such layers and precisely adjust their influence (see \Cref{fig:workflow}).

\paragraph{Mutual Information Analysis} We hypothesize to find consistent, concept-dependent patterns of distributed layer responsibilities in T2I models that could be exploited for concept-specific guidance. To validate our intuition and to provide insight into \textit{why} guiding with skipped layers works \cite{hyung2025spatiotemporal,huggingfaceSD3}, we analyze the mutual information (MI) between each layer and a given concept. We adapt the work of Wang et al. \cite{wang2025information}, where MI is used to locate layers that are responsible for general text-to-image alignment. Wang et al. \cite{wang2025information} formulate their notion as the expected difference between conditional and unconditional predictions:

\begin{equation}
    I(x,c) = \mathbb{E}_{t, \epsilon} \kappa_t \left\| \epsilon_\theta(x_t, t, c) - \epsilon_\theta(x_t, t, \emptyset) \right\|^2, 
    \label{eq:mi}
\end{equation}
where $\kappa_t$ scales the contribution of each timestep, reflecting information flow at that denoising stage. To extend this framework for our method, we introduce per-layer MI by computing \Cref{eq:mi} with a single skipped layer: 

\begin{equation}
    I(x,c,i) = \mathbb{E}_{t, \epsilon} \kappa_t \left\| \epsilon_{[\theta \setminus i]}(x_t, t, c) - \epsilon_{[\theta \setminus i]}(x_t, t, \emptyset) \right\|^2,
\end{equation}
where we denote the conditional noise prediction with layer $i$ skipped as $\epsilon_{[\theta \setminus i]}(x_t, t, c)$, or simply $\epsilon_{[\theta \setminus i]}$. We then define per-concept, per-layer MI as the difference in MI given a positive and a negative text prompt $c$ and $c \setminus y$ that differ only in the presence of the target $y$. That is, the mutual information between the target concept $y$ and images generated while skipping layer $i$ is given by

\begin{equation}
    I(x,y,i) = I(x,c,i) - I(x,c \setminus y,i).
    \label{eq:mi_concept}
\end{equation}

\begin{figure}[t]
    \centering
    \begin{minipage}[b]{0.45\linewidth}
        \centering
        \includegraphics[width=0.95\linewidth]{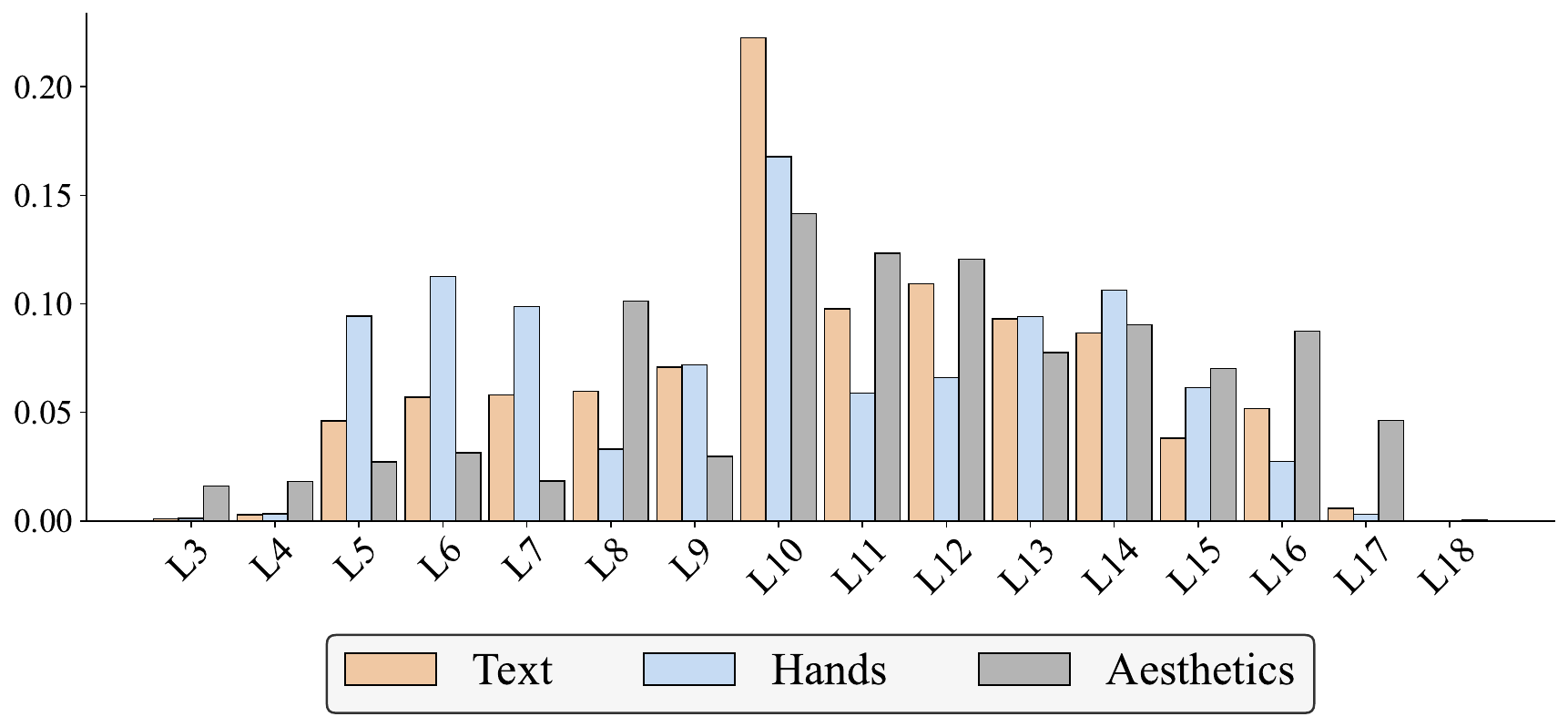}
        \par\smallskip (a) FLUX.1-dev
    \end{minipage}
    \hfill
    \begin{minipage}[b]{0.45\linewidth}
        \centering
        \includegraphics[width=0.95\linewidth]{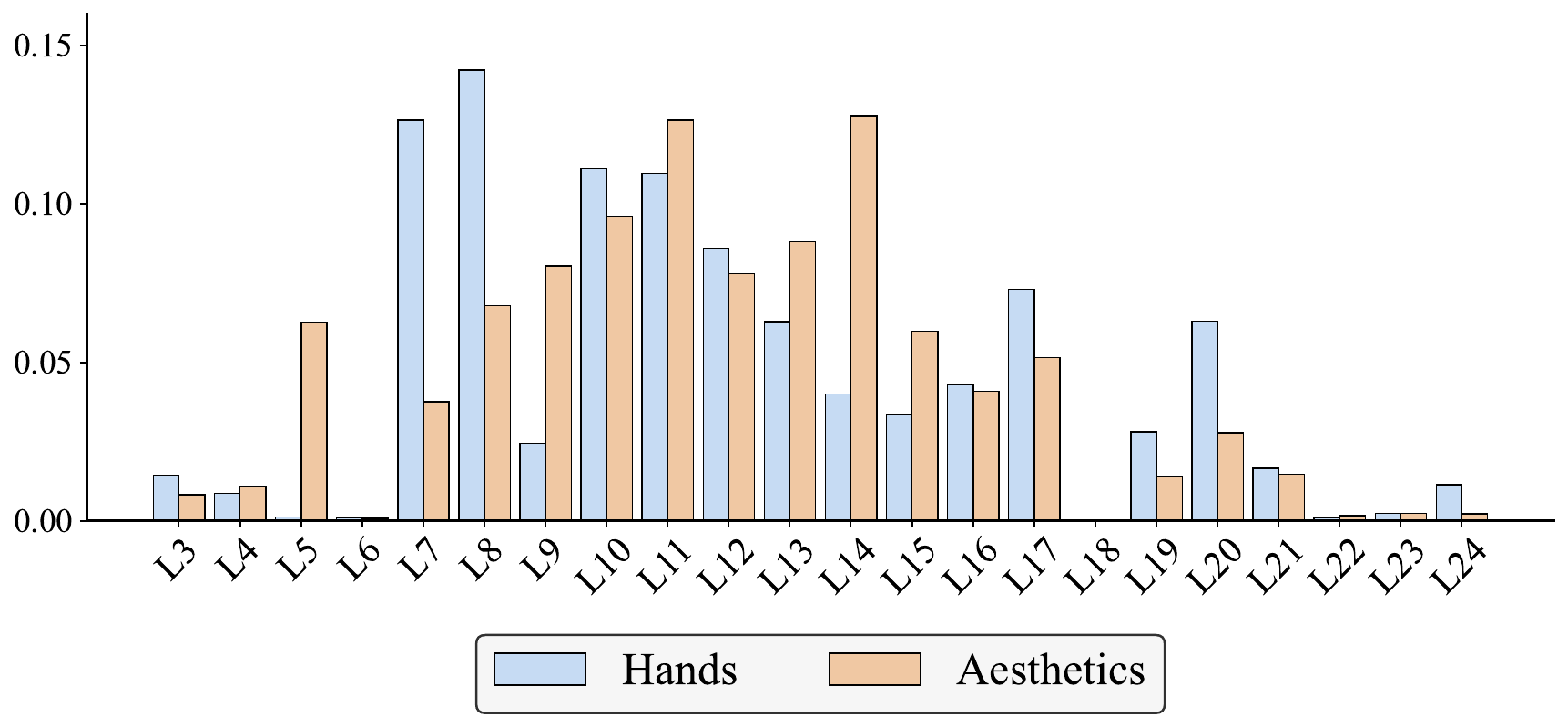}
        \par\smallskip (b) PixArt-$\alpha$
    \end{minipage}
    \caption{\textbf{Different concepts concentrate in different layers, enabling concept-specific layer selection for guidance.}
Bar plots report per-layer mutual information in (a) FLUX.1-dev and (b) PixArt-$\alpha$, revealing strong, concept-dependent variation.}
    \label{fig:mi_combined} % A new, combined label for the entire figure
\end{figure}
\paragraph{Layer-Skipping} We skip concept-relevant layers by redefining the residual mapping as an identity function for each skipped layer, similar to STG \cite{hyung2025spatiotemporal}: 

\begin{align}
\text{Res}(z_l) = f_l(z_l) + z_l, \quad \overline{\text{Res}}(z_l) = id(z_{l}) = z_l,
\label{eq:skip}
\end{align}
where $z_l$ denotes the feature representation at the $l$-th layer, and $f_l$ represents the nonlinear transformation in layer $l$. In $\overline{\text{Res}}$, the block output is set equal to its input, thus bypassing $f_l$. This preserves information flow while preventing additional perturbations, allowing for controlled modulation.

\paragraph{Locating Layer-Directions in Latent Space} The fundamental intuition behind our method is to decompose layer-wise predictions into a target direction $\vec{y}$ and a residual $\vec{r}$. Let $d_i$ be the noise prediction with layer $i$ amplified, then

\begin{equation}
    d_i = \alpha_i\vec{y} + \vec{r}_i,
\end{equation}
where \(\alpha_i\ge 0\), and \(r_i\) satisfies \(\langle r_i,y\rangle = 0\). Then, the idea is to extract information about the magnitude of $\alpha_i$ by profiling the effectiveness of single-layer skip-guidance. For each layer $i$, we generate a set of images by using the skipped prediction $\epsilon_{[\theta \setminus i]}$ as a negative direction to guide away from, similar to the unconditional noise prediction in \Cref{eq:cfg}. We measure the target performance $p_i$ of these resulting images. Given this performance observation, the per-layer performance can be represented as a sum of \(\alpha_i\) plus a noise component $\epsilon_i$:

\begin{equation}
    p_i = \alpha_i + \epsilon_i,
\end{equation}
where we assume that \(\{\epsilon_i\}_{i\ge1}\) are i.i.d. with \(\mathbb E[\epsilon_i]=0\), and that \(\epsilon_i\) is independent of \((\alpha_i,\vec{r}_i)\). Observing performances then yields a per-layer impact distribution for a given concept. The computational cost of this procedure scales linearly with the number of layers and the number of samples, i.e., $\mathcal{O}(L \cdot N)$ for $L$ layers and $N$ samples per layer. Notably, profiling is performed only once per model and concept. We make our layer analysis and code available at \url{https://github.com/CompVis/concept\_guidance}. Pseudocode  is provided in \Cref{alg:msg_profiling,alg:msg_inference}.

\begin{figure}[t]
\centering
\includegraphics[width=\textwidth]{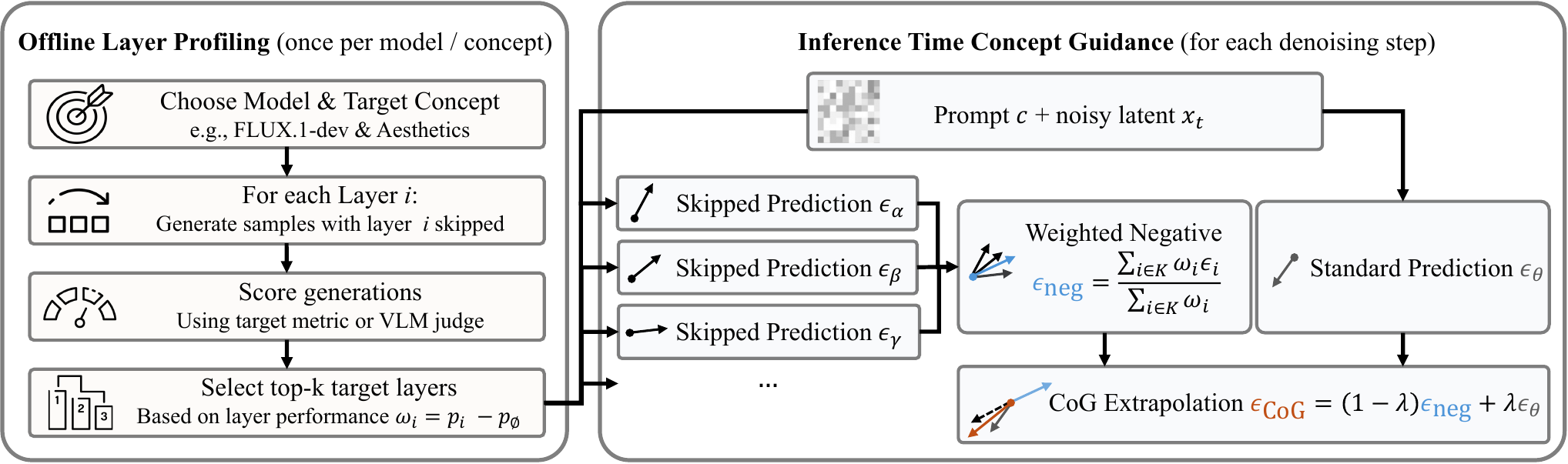}
\caption{\textbf{Concept Guidance Workflow.} \textit{Left:} In a one-time profiling stage, each layer is skipped individually, generations are scored, and top-$k$ layers are selected. \textit{Right:} At inference, per-layer skipped predictions are aggregated into a performance-weighted negative prediction; CoG then extrapolates from beyond the standard prediction.}
\label{fig:workflow}
\end{figure}

\paragraph{Performance-Weighted Multi-Layer Guidance} To approximate $\vec{y}$, we compute predictions for the top-$k$ best performing layers and weigh their impact on the overall prediction using their observed performance $p_i$. Let $K$ be the set of best-performing layers, then we compute an individual negative noise prediction with layer $i$ skipped for all $i \in K$. Each prediction is weighed by the performance term $p_i$ relative to the target performance without layer-skipping $p_\emptyset$. Specifically, the weight $\omega_i$ for the prediction with layer $i$ skipped is given by: 

\begin{equation}
    \omega_i = p_i - p_\emptyset.
\end{equation}

\begin{figure}[t]
\centering
\includegraphics[width=0.95\textwidth]{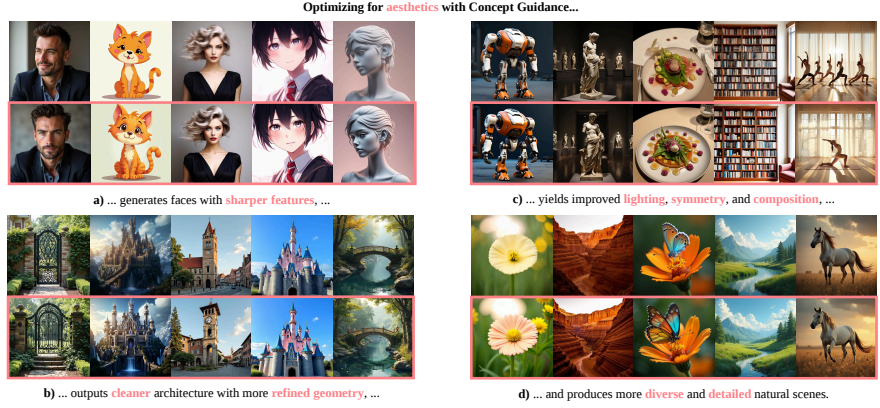}
\caption{\textbf{Optimizing aesthetics with CoG improves perceptual quality.} Side-by-side examples compare CFG vs. CoG.}
\label{fig:aesthetics}
\end{figure}

Thus, each layer contributes to the guidance process only insofar as it improves target performance compared to standard generation. The final negative noise prediction is then given by a weighted mean over all $k$ noise predictions: 

\begin{equation}
    \epsilon_{neg} = \frac{\sum_{i \in K} \omega_i \cdot \epsilon_{[\theta \setminus i]}}{\sum_{i \in K}\omega_i} .
\label{eq:msg}
\end{equation}
Finally, we extrapolate $\epsilon_{neg}$ beyond the standard noise prediction $\epsilon_\theta$: 

\begin{equation}
    \epsilon_{CoG} = (1-\lambda)\epsilon_{neg} + \lambda \epsilon_\theta, 
\label{eq:msg_guidance}
\end{equation}
 where $\epsilon_\theta$ is given by CFG according to \Cref{eq:cfg}, and $\lambda \geq 1$ controls the guidance strength. Thus, CoG integrates seamlessly with CFG and allows easily combining general prompt adherence with target-specific guidance.

\section{Experiments}
\label{sec:experiments}

\begin{figure*}[t]
\centering
\includegraphics[width=0.95\textwidth]{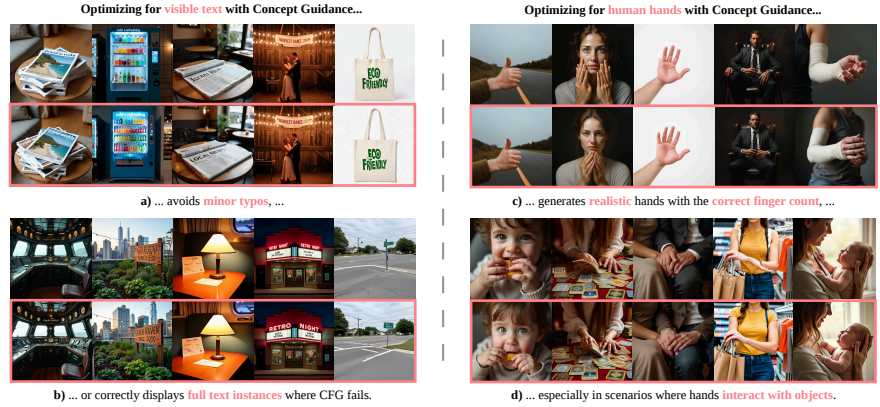}
\caption{\textbf{CoG improves fine-grained local coherence for hard failure cases like visible text and hands.}
Side-by-side examples compare CFG vs. CoG.}
\label{fig:texthands}
\end{figure*}

\subsection{Setup}
We evaluate our method on multiple models with varying architecture, size, and output quality, such as \textbf{PixArt-$\boldsymbol{\alpha}$} \cite{chen2023pixart}, \textbf{Stable Diffusion 3} \cite{esser2024scaling, huggingfaceSD3}, \textbf{Stable Diffusion 3.5} \cite{huggingfaceSD35large} and \textbf{FLUX.1-dev} \cite{huggingfaceFLUX}. As targets we choose text and hand generation since text-to-image diffusion models notoriously struggle with those concepts. To provide evidence that CoG also extends to more abstract concepts we choose general aesthetics. We evaluate with EasyOCR \cite{vedhaviyassh2022comparative}, a pretrained model from MediaPipe \cite{lugaresi2019mediapipe} and a pretrained model for evaluating aesthetics \cite{schuhmann2022laion} based on CLIP embeddings \cite{radford2021learning}. Furthermore, we investigate if improving a target concept with CoG creates a trade-off for overall image quality using HPSv3 \cite{ma2025hpsv3}, which aligns strongly with human preferences. Throughout, all guidance methods are evaluated on identical prompts, seeds, and scheduler settings, so reported improvements correspond to paired comparisons. Beyond these metrics, we show that CoG also generalizes to concepts without a hand-crafted metric by using a vision-language model as the scoring judge (\Cref{sec:vlm}), and we release our per-model, per-concept layer configurations (\Cref{tab:combined_layers}).

\subsection{Layer Analysis}

We derive insights on layer-concept interactions from computing per-layer, per-concept MI according to \Cref{eq:mi_concept}. Evaluating MI across all tasks for FLUX.1-dev and PixArt-$\alpha$
(see \Cref{fig:mi_combined} (a) and (b)) reveals that (i) there exist large concept-dependent differences, which supports our claim that networks distribute concept-specific knowledge non-uniformly across layers and (ii) we find that layers with high per-concept MI are often located in the middle of the network.

To validate our findings, we evaluate correlations between per-layer MI and the measured layer performance on the target metric, and find very high correlations in the range of $[0.692, 0.956]$. This strongly reinforces our intuition that performance-weighted, multi-skip Concept Guidance works by extrapolating the influence of the layers that are most responsible for the generation of target concepts.

\subsection{Qualitative Results} CoG consistently improves visual quality across all concepts. We present comparisons of CFG and CoG in \Cref{fig:aesthetics} and \Cref{fig:texthands}. We further provide extensive uncurated comparisons of CoG and several baselines in \Cref{fig:text_uncurated,fig:aesth_uncurated}.

\begin{table}[t]
    \centering
    
    \begin{minipage}[t]{0.505\textwidth}
        \centering
        \small
        \caption{\textbf{Concept Guidance achieves better performance across measured tasks.} $^{\dagger}$ does not generate visible text.}
        \setlength{\tabcolsep}{4pt} % Reduced slightly to fit better in the minipage
        \adjustbox{max width=\linewidth}{
        \begin{tabular}{l l *{3}{c}}
        \toprule
        \textbf{Model} & \textbf{Method} & \textbf{Text} ($\uparrow$) & \textbf{Hands} ($\uparrow$) & \textbf{Aesthetics} ($\uparrow$) \\
        \midrule
        
        \multirow{3}{*}{PixArt-$\alpha$$^{\dagger}$}
        & CFG & - \hspace{3mm} & 0.357 & 6.667 \\
        & CoG$_{single}$ & - \hspace{3mm} & 0.496  & 6.709  \\
        & CoG$_{multi}$  & - \hspace{3mm} & \textbf{0.509}\rlap{ \perc{42\%}} & \textbf{6.722}\rlap{ \perc{0.8\%}} \\
        \midrule
        
        \multirow{3}{*}{SD3}
        & CFG & 0.476 \hspace{3mm} & 0.651 & 6.197 \\
        & CoG$_{single}$  & 0.483 \hspace{3mm} & 0.674  & 6.333  \\
        & CoG$_{multi}$ & \textbf{0.490}\rlap{ \perc{2.9\%}} \hspace{3mm} & \textbf{0.684}\rlap{ \perc{5.0\%}} & \textbf{6.348}\rlap{ \perc{2.4\%}} \\
        \midrule
        
        \multirow{3}{*}{SD3.5}
        & CFG & 0.580 \hspace{3mm} & 0.646 & 6.276 \\
        & CoG$_{single}$  & 0.598 \hspace{3mm} & 0.662 & \textbf{6.408}\rlap{ \perc{2.1\%}} \\
        & CoG$_{multi}$ & \textbf{0.610}\rlap{ \perc{5.2\%}} \hspace{3mm} & \textbf{0.679}\rlap{ \perc{5.1\%}} & 6.368 \\
        \midrule
        
        \multirow{3}{*}{FLUX.1-dev}
        & CFG & 0.481 \hspace{3mm} & 0.672 & 6.448 \\
        & CoG$_{single}$  & 0.499  \hspace{3mm} & 0.655  & 6.631 \\
        & CoG$_{multi}$ & \textbf{0.510}\rlap{ \perc{6.1\%}} \hspace{3mm} & \textbf{0.675}\rlap{ \perc{0.5\%}} & \textbf{6.657}\rlap{ \perc{3.2\%}} \\
        \bottomrule
        \end{tabular}
        }
        \label{tab:main_results}
    \end{minipage}\hfill % This adds flexible spacing between the two minipages
    \begin{minipage}[t]{0.475\textwidth}
        \centering
        \caption{\textbf{Concept Guidance can be applied to multiple concepts at once.} $^{\dagger}$ does not generate visible text.} 
        \adjustbox{max width=\linewidth}{
        \begin{tabular}{l l *{6}{c}} % 1 (Model) + 1 (Method) + 6 (Scores) = 8 cols
        \toprule
        & & \multicolumn{2}{c}{\textbf{\begin{tabular}{@{}c@{}}Text ($\uparrow$)  + \\ Hands ($\uparrow$) \end{tabular}}}
          & \multicolumn{2}{c}{\textbf{\begin{tabular}{@{}c@{}}Text ($\uparrow$) + \\ Aesthetics ($\uparrow$) \end{tabular}}}
          & \multicolumn{2}{c}{\textbf{\begin{tabular}{@{}c@{}}Hands ($\uparrow$) + \\ Aesthetics ($\uparrow$) \end{tabular}}} \\
        \cmidrule(lr){3-4} \cmidrule(lr){5-6} \cmidrule(lr){7-8}
        \textbf{Model} & \textbf{Method} & \textbf{Text} & \textbf{Hands} & \textbf{Text}& \textbf{Aes.} & \textbf{Hands} & \textbf{Aes.} \\
        
        \midrule
        \multirow{3}{*}{PixArt-$\alpha$$^{\dagger}$}
        & CFG & - & - & - & - & 0.330 & 6.217 \\
        & CoG$_{single}$ & - & - & - & - & 0.252 &  6.231\\
        & CoG$_{multi}$ & - & - & - & - & \textbf{0.334} &  \textbf{6.260}\\
        
        \midrule
        
        \multirow{3}{*}{SD3}
        & CFG &  0.353    &  0.423    &  0.416    &  6.213    &  0.631    &  5.922    \\
        & CoG$_{single}$ &  0.354    &  0.502    &  0.396    &  6.222    &  0.720    &  5.915    \\
        & CoG$_{multi}$  & \textbf{ 0.356   } & \textbf{ 0.503   } & \textbf{ 0.424   } & \textbf{ 6.282   } & \textbf{ 0.745   } & \textbf{ 5.938   } \\
        \midrule
        
        \multirow{3}{*}{SD3.5}
        & CFG &    0.498    &  0.620    &  0.637    &  \textbf{5.974}    & 0.330 & 6.217 \\
        & CoG$_{single}$ &  0.493    &  0.655    &  0.634    &  5.861    & 0.252 & 6.260 \\
        & CoG$_{multi}$ &  \textbf{0.508}    & \textbf{ 0.693   } & \textbf{ 0.642   } &  5.893   & \textbf{0.334} & \textbf{6.231} \\
        
        \midrule
        
        \multirow{3}{*}{FLUX.1-dev}
        & CFG &  0.346 & \textbf{0.554} & 0.377 & 6.606 & 0.668 & 6.205 \\
        & CoG$_{single}$~& 0.354 & 0.494 & 0.418 & 6.628 & 0.677 & 6.215 \\
        & CoG$_{multi}$ & \textbf{0.377} & 0.538 & \textbf{0.431} & \textbf{6.657} & \textbf{0.710} & \textbf{6.301} \\
        
        \bottomrule
        \end{tabular}
        }
        \label{tab:combined_results_multicolumn}
    \end{minipage}
\end{table}

\paragraph{Complex Compositional Tasks.} Concept Guidance improves tasks that require fine-grained structural coherence and correct object interactions. For \textit{human hands}, CoG visibly reduces common artifacts and produces more plausible results, particularly in complex scenarios where hands interact with other objects (\Cref{fig:texthands} \textit{d}). When generating \textit{visible text}, CoG consistently generates text more faithful to the prompt. In some cases, CoG even successfully renders complete and correct text where CFG fails to produce any readable output (\Cref{fig:texthands} \textit{b}).

\paragraph{General Concepts.} We find that Concept Guidance also improves concept-specific alignment beyond notorious failure cases of T2I models. When optimizing for the more general concept of \textit{aesthetics}, CoG yields both general and prompt-specific enhancements. Overall, we find that CoG produces images with improved lighting, contrast, and compositional symmetry (\Cref{fig:aesthetics} \textit{c}). Moreover, CoG adapts stylistic elements to the prompt. Human faces exhibit sharper features, and healthier skin tones (\Cref{fig:aesthetics} \textit{a}). Architectural scenes appear more modern, clean, and luxurious (\Cref{fig:aesthetics} \textit{b}). Lastly, natural landscapes display greater visual diversity, for instance, through a richer variety of vegetation (\Cref{fig:aesthetics} \textit{d}).

\subsection{Quantitative Results} We evaluate CoG with multiple performance-weighted layer skips (CoG$_{multi}$) using four models and three different target concepts, and compare our method against Classifier-Free Guidance \cite{ho2022classifier_free_guidance} and Concept Guidance with only a single, target-optimized skipped layer (CoG$_{single}$). Across all settings and models, CoG$_{multi}$ consistently outperforms both CFG and CoG$_{single}$ (see \Cref{tab:main_results}): Concept Guidance achieves an average performance increase of $8.1\%$ compared to CFG, and at its best, a $42\%$ increase for hand generation with PixArt-$\alpha$.

We also find consistent improvements across models. CoG achieves an average target performance increase of $3.3\%$ for FLUX.1-dev \cite{huggingfaceFLUX}, $3.5\%$ for SD3 \cite{huggingfaceSD3}, $3.9\%$ for SD3.5 \cite{huggingfaceSD35large}, and $22\%$ for PixArt-$\alpha$ \cite{chen2023pixart}. 
Regarding different tasks, CoG yields stable but moderate increases for aesthetics, and higher increases for tasks that require high local coherence. 
Specifically, CoG increases target performance by $2.0\%$ for aesthetics, by $4.7\%$ 
for text, and by $13.4\%$ for hands. We suspect localized tasks are especially sensitive to guidance accuracy. Our intuition is that CoG's improvements are caused by a better-aligned update direction.

\begin{figure}[t]
    \centering
    \includegraphics[width=0.7\textwidth]{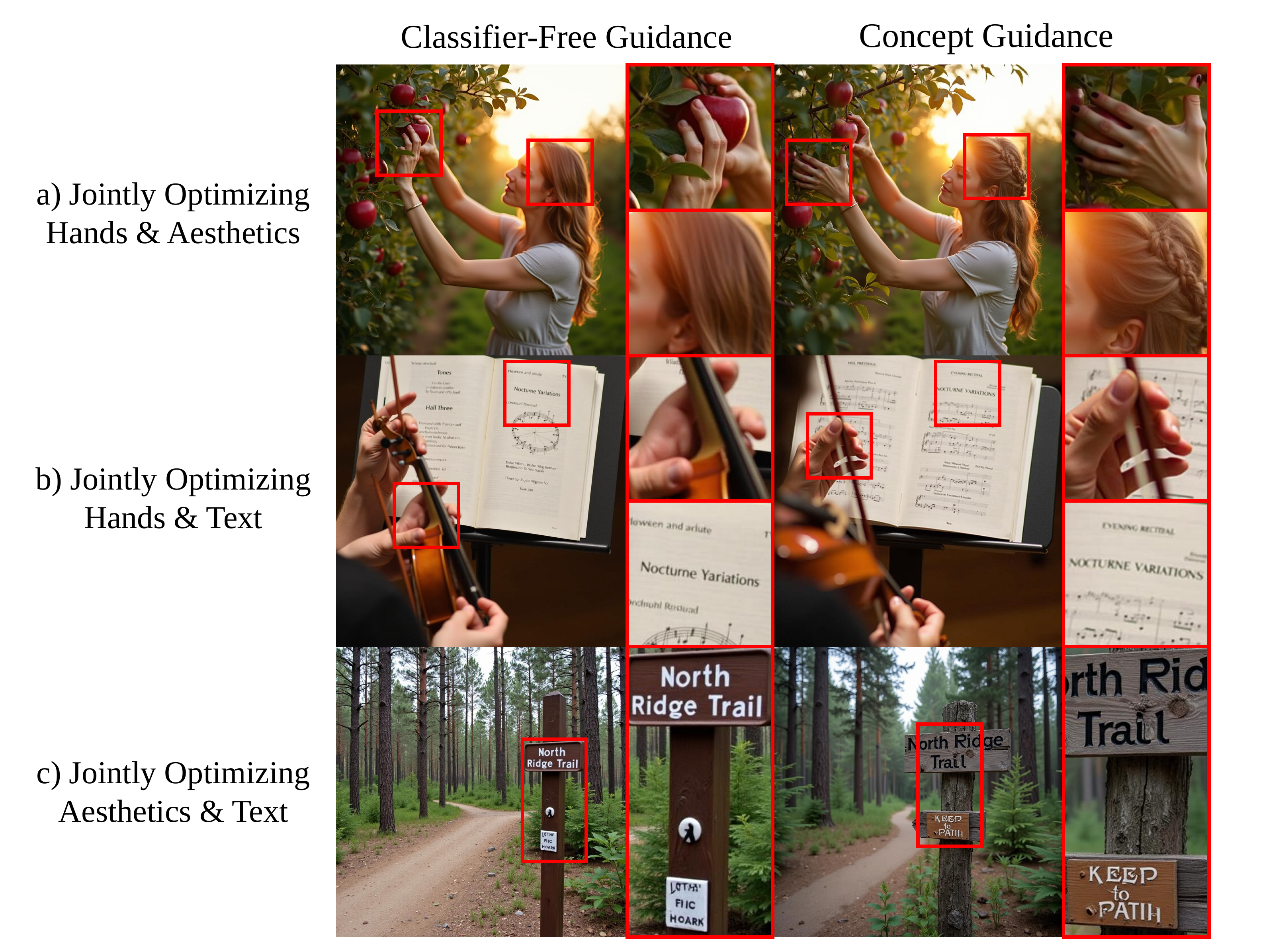}
    \caption{\textbf{Multi-target Concept Guidance} simultaneously improves multiple concepts on FLUX.1-dev compared to CFG.}
    \label{figure:combined}
\end{figure}

\paragraph{Combined Targets.} We then evaluate CoG on multiple targets at once. Here, we only use layers for skipping that increase performance for all targets, and choose those that yield the highest combined gain. We find that while synergies between targets vary, CoG again consistently outperforms CFG (see \cref{tab:combined_results_multicolumn}). Using only a single layer for guidance is often worse than CFG, indicating that weighted skipping of multiple layers is essential for optimizing several targets. We show qualitative results in \Cref{figure:combined}, where CoG jointly resolves common failure cases.

\begin{table}[b]
    \centering
    \setlength{\tabcolsep}{1.2pt}
    \renewcommand{\arraystretch}{.80}
    \caption{\textbf{Evaluation Metrics.} Comparison against APG \cite{sadat2024eliminating} and PAG \cite{ahn2024self}.}
    
    \begin{minipage}[t]{0.48\linewidth}
        \centering
        \textbf{(a) SD3 \& SD3.5}

        \adjustbox{max width=\linewidth}{%
        \begin{tabular}{@{}llccc@{}}
            \toprule
            \textbf{Model} & \textbf{Method} & \textbf{Aesth.} & \textbf{Hands} & \textbf{Text} \\
            \midrule
            \multirow{5}{*}{\makebox[0pt][l]{SD3}\hphantom{PixArt-$\alpha$}}
            & CFG & 6.358 & 0.664 & 0.477 \\
            & CFG+APG \cite{sadat2024eliminating} & 6.524 & 0.590 & 0.396 \\
            & PAG \cite{ahn2024self} & 6.354 & 0.687 & 0.481 \\
            \cmidrule(lr){2-5}
            & CoG & 6.494 & 0.707 & \textbf{0.491} \\
            & CoG+APG & \textbf{6.600} & \textbf{0.746} & 0.410 \\
            \midrule
            \multirow{5}{*}{\makebox[0pt][l]{SD3.5}\hphantom{PixArt-$\alpha$}}
            & CFG & 6.293 & 0.475 & 0.504 \\
            & CFG+APG \cite{sadat2024eliminating}  & 6.462 & 0.535 & 0.463 \\
            & PAG \cite{ahn2024self} & 6.200 & 0.352 & 0.491 \\
            \cmidrule(lr){2-5}
            & CoG & 6.344 & 0.618 & \textbf{0.514} \\
            & CoG+APG & \textbf{6.499} & \textbf{0.733} & 0.476 \\
            \bottomrule
        \end{tabular}}
    \end{minipage}\hfill
    \begin{minipage}[t]{0.48\linewidth}
        \centering
        \textbf{(b) PixArt \& FLUX}

        \adjustbox{max width=\linewidth}{%
        \begin{tabular}{@{}llccc@{}}
            \toprule
            \textbf{Model} & \textbf{Method} & \textbf{Aesth.} & \textbf{Hands} & \textbf{Text} \\
            \midrule
            \multirow{5}{*}{PixArt-$\alpha$}
            & CFG & 6.622 & 0.388 & -- \\
            & CFG+APG \cite{sadat2024eliminating} & 7.096 & 0.399 & -- \\
            & PAG \cite{ahn2024self} & 6.606 & 0.401 & -- \\
            \cmidrule(lr){2-5}
            & CoG & 6.658 & 0.437 & -- \\
            & CoG+APG & \textbf{7.143} & \textbf{0.499} & -- \\
            \midrule
            \multirow{5}{*}{FLUX.1-dev}
            & CFG & 6.585 & 0.727 & 0.608 \\
            & CFG+APG \cite{sadat2024eliminating} & 6.650 & 0.747 & 0.472 \\
            & PAG \cite{ahn2024self} & 6.652 & 0.626 & 0.629 \\
            \cmidrule(lr){2-5}
            & CoG & \textbf{6.659} & \textbf{0.754} & \textbf{0.650} \\
            & CoG+APG & 6.573 & 0.738 & 0.586 \\
            \bottomrule
        \end{tabular}}
    \end{minipage}
    \label{tab:main_metrics}
\end{table}

\begin{table}[t]
    \begin{minipage}[t]{0.55\textwidth}
        \centering
        \setlength{\tabcolsep}{1.2pt}
        \renewcommand{\arraystretch}{.80}
        
        \caption{\textbf{Auxiliary Metrics.}}

        \adjustbox{max width=\linewidth}{%
        \begin{tabular}{@{}llccc@{}}
            \toprule
            \textbf{Model} & \textbf{Method} & \textbf{KDD}$\downarrow$ & \textbf{CLIP}$\uparrow$ & \textbf{LPIPS}$\uparrow$ \\
            \midrule
            \multirow{4}{*}{PixArt-$\alpha$}
            & CFG & \textbf{1.589} & \textbf{31.05} & 0.649 \\
            & CoG$_{text}$ & -- & -- & -- \\
            & CoG$_{hands}$ & 2.127 & 30.81 & 0.584 \\
            & CoG$_{aesth.}$ & 1.679 & 30.98 & \textbf{0.655} \\
            \midrule
            \multirow{4}{*}{\makebox[0pt][l]{SD3}\hphantom{PixArt-$\alpha$}}
            & CFG & \textbf{0.869} & \textbf{31.97} & 0.635 \\
            & CoG$_{text}$ & 0.961 & 31.77 & 0.650 \\
            & CoG$_{hands}$ & 0.967 & 31.77 & 0.645 \\
            & CoG$_{aesth.}$ & 0.955 & 31.69 & \textbf{0.653} \\
            \bottomrule
        \end{tabular}}
        \label{tab:aux_metrics}
    \end{minipage}\hfill
    \begin{minipage}[t]{0.45\textwidth}
        \centering
        \caption{\textbf{Human Preference Scores favor CoG over CFG on SD3.5 \cite{huggingfaceSD35large}.} Especially when optimizing for aesthetics, CoG yields preferred images.}
        \adjustbox{max width=\linewidth}{%
        \begin{tabular}{l *3{w{c}{2.0cm}} w{c}{2.0cm}}
            \toprule
             & \textbf{Text} ($\uparrow$) & \textbf{Hands} ($\uparrow$)& \textbf{Aesthetics} ($\uparrow$) & \textbf{Mean} ($\uparrow$) \\
            \midrule
            CFG \cite{ho2022classifier_free_guidance} & \textbf{8.991} & 6.298 & 7.832 & 7.707 \\
            CoG$_{\mathbf{multi}}$                     & 8.945 & \textbf{6.552} & \textbf{8.277} & \textbf{7.925} \\
            \midrule
            Win Rate                                  & 48.4\% & 62.6\% & 74.6\% & 61.87\% \\
            \bottomrule
        \end{tabular}%
        }
        \label{tab:human_preference}
    \end{minipage}
\end{table}

\smallskip \noindent \textit{Stronger Baselines.} We further compare CoG against two state-of-the-art training-free guidance methods: Adaptive Projected Guidance (APG) \cite{sadat2024eliminating} and Perturbed-Attention Guidance (PAG) \cite{ahn2024self}. As APG and PAG improve \emph{general} guidance behavior while CoG contributes a \emph{concept-aware} guidance direction, the two are complementary and can be combined (CoG+APG). As shown in \Cref{tab:main_metrics}~(a, b), CoG improves over CFG in every setting, and either CoG or CoG+APG is best overall in every cell. Complementarity is clearest for aesthetics and hands, where CoG+APG wins, while CoG alone is better for text.

\smallskip \noindent \textit{Quality and Diversity.} We emphasize that CoG does not aim to improve \emph{general} generation quality, but rather provides systematic, concept-aware inference-time steering; we therefore report auxiliary metrics to characterize the trade-off between general and concept-specific quality. We use Kernel DINO Distance (KDD) instead of FID, as FID is poorly aligned with perceptual quality for state-of-the-art models \cite{stein2023exposing, yang2026representation}, as well as CLIP-Score and LPIPS (\Cref{tab:aux_metrics}). CLIP-Score remains close to CFG, LPIPS shows no systematic diversity collapse, and KDD reflects the expected fidelity/control trade-off.
Consistently, the HPSv3 scores reported below measure \emph{overall} generation quality while steering toward a concept, rather than concept quality itself.

\smallskip \noindent \textit{Human Preference Comparison.} We use HPSv3 \cite{ma2025hpsv3} as a preference-based proxy to probe perceptual trade-offs under target optimization and present absolute values and win-rates vs. CFG in \cref{tab:human_preference}. We explicitly do\emph{ not} optimize Concept Guidance for HPSv3: we keep the standard, per-target settings and evaluate $5{,}000$ samples per target on SD3.5 \cite{huggingfaceSD35large}. CoG is preferred for Hands (62.6\%) and Aesthetics (74.6\%), indicating that CoG is preferred by humans, especially when optimizing aesthetics. Text is near parity, indicating a minor trade-off.

\subsection{Ablation Studies}
\label{sec:ablations}

Here we summarize three ablations; the detailed discussion, figures (\Cref{fig:numlayers_ablation,fig:multiskip_ablation}) and tables are provided in \Cref{sec:ablation_appendix}. \textbf{(i)~Per-concept layers:} selecting skipped layers per concept, rather than using a single fixed layer as in STG \cite{hyung2025spatiotemporal}, improves target performance by up to $24.5\%$ and on average $5.6\%$ (\Cref{tab:stg_vs_cg}). \textbf{(ii)~Number of skipped layers:} performance improves with $k$ up to a sweet spot around $k=2$--$3$, with most of the gain already obtained from a single relevant layer, while larger $k$ eventually introduces interference between layer directions (\Cref{fig:numlayers_ablation}, \Cref{tab:num_layers_ablation}); the additional inference cost is therefore opt-in (\Cref{sec:cost}). \textbf{(iii)~Aggregation strategy:} performance-weighted aggregation of separate per-layer predictions outperforms both \emph{Naive} single-pass multi-skipping (as in the HuggingFace SD3 pipeline \cite{huggingfaceSD3}) and \emph{Uniform} weighting (\Cref{fig:multiskip_ablation}, \Cref{tab:guidance_methods}), suggesting that popular diffusion pipelines could be improved by incorporating our method. Finally, CoG is robust to the choice of guidance scale $\lambda$ (\Cref{tab:scale_ablation}).

\paragraph{Limitations.} Concept Guidance applies to concepts with a meaningful scoring signal: either an automatic metric or a VLM-based judge (\Cref{sec:vlm}). Layer profiling is performed once per model/concept pair and should not be assumed to transfer across backbones, so rankings must be recomputed for new models (we release our configurations in \Cref{tab:combined_layers} to avoid this cost for the models we study). At inference, CoG adds one noise prediction per skipped layer, increasing latency with $k$ (\Cref{sec:cost}). In practice, a single layer already yields most of the benefit. Finally, CoG steers a targeted concept rather than improving general generation quality resulting in a potential trade-off.

\section{Conclusion}
We introduced Concept Guidance, a simple, general, and effective mechanism for precise latent control in text-to-image diffusion models. By identifying and amplifying the influence of concept-specific layers, CoG solves persistent, well-known failure cases of standard guidance--like misspelled text and malformed hands -- but can also optimize for more general concepts like overall aesthetics. Concept Guidance's key strength is its usability: it is a plug-and-play component that integrates seamlessly with CFG, requires no training, gradients, or external models, and can be added to any existing pipeline with minimal modification.

\enlargethispage{2\baselineskip}
\begin{credits}
\subsubsection{\ackname}
This work has been supported by the Horizon Europe project ELLIOT (GA No. 101214398), the German Federal Ministry for Economic Affairs and Energy within the project “NXT GEN AI METHODS – Generative Methoden für Perzeption, Prädiktion und Planung”, the project “GeniusRobot” (01IS24083) funded by the Federal Ministry of Research, Technology and Space (BMFTR), and the BMWE ZIM-project (No. KK5785001LO4) “conIDitional LoRA”. The authors gratefully acknowledge the Gauss Center for Supercomputing for providing compute through the NIC on JUWELS/JUPITER at JSC and the HPC resources supplied by the NHR@FAU Erlangen. Furthermore, this work was partially supported by the DAAD programme Konrad Zuse Schools of Excellence in Artificial Intelligence, sponsored by the Federal Ministry of Research, Technology and Space.
\end{credits}

\clearpage
\bibliographystyle{splncs04}
\bibliography{CoG_arxiv}

\clearpage
\appendix

\begin{center}
   \Large \bfseries Supplementary Material \\
\end{center}

\setcounter{section}{0}
\renewcommand{\thesection}{\Alph{section}} 

\section{Layer Indices and Weights}
\label{sec:layer_indices}

To facilitate reproducibility and allow future work to apply Concept Guidance without the computational cost of layer profiling, we provide the exact configurations used in our experiments. For each model and target concept, we list the set of top-$k$ layer indices $\mathcal{K}$ identified as most impactful, along with their corresponding importance weights $\omega_i$. As defined in \Cref{sec:msg}, the weights represent the performance gain of the skipped layer relative to the baseline.

\begin{table}[h]
\centering
\setlength{\aboverulesep}{0pt}
\setlength{\belowrulesep}{0pt}
\renewcommand{\arraystretch}{1.25}
\caption{\textbf{Layer Configurations.} Top-$k$ skipped layers ($\mathcal{K}$) and their corresponding weights ($\Omega$) for all evaluated models and target concepts.}
\resizebox{0.7\columnwidth}{!}{%
\begin{tabular}{@{} l l c c c @{}}

\toprule
\textbf{Model} & \textbf{Task} & \textbf{$k$} & \textbf{Skipped Layers ($\mathcal{K}$)} & \textbf{Weights ($\Omega$)} \\
\midrule
\multirow{2}{*}{PixArt-$\alpha$} 
 & Human Hands & [4] & [3, 23, 8, 18] & [1.00, 0.88, 0.73, 0.63] \\
 & Aesthetics & [3] & [9, 20, 5] & [1.00, 0.28, 0.12] \\
\midrule
\multirow{3}{*}{SD3} 
 & Visible Text & [3] & [7, 11, 4] & [1.00, 0.55, 0.45] \\
 & Human Hands & [4] & [9, 6, 5, 10] & [1.00, 0.48, 0.45, 0.21] \\
 & Aesthetics & [3] & [6, 7, 8] & [1.00, 0.71, 0.70] \\
\midrule
\multirow{3}{*}{SD3.5} 
 & Visible Text & [2] & [15, 10] & [1.00, 0.49] \\
 & Human Hands & [2] & [7, 18] & [1.00, 0.98] \\
 & Aesthetics & [1] & [4] & [1.00] \\
\midrule
\multirow{3}{*}{FLUX.1-dev} 
 & Visible Text & [3] & [11, 3, 9] & [1.00, 0.85, 0.74] \\
 & Human Hands & [3] & [9, 12, 15] & [1.00, 0.81, 0.78] \\
 & Aesthetics & [3] & [15, 14, 13] & [1.00, 0.18, 0.09] \\
\bottomrule
\end{tabular}
}

\label{tab:combined_layers}
\end{table}

\section{Convergence Analysis}

Empirically, as demonstrated in \Cref{sec:ablations}, we find that CoG approximates the latent space direction of highest target performance $\vec{y}$ more precisely for a larger $k$, at least until a certain threshold. Here, we formalize our intuition and prove that under \emph{idealized} assumptions, CoG predictions converge to $\vec{y}$ as the number of skipped layers goes to infinity. In practice, the number of skippable layers is of course bounded by the model architecture.

\begin{lemma}
Let $d_i = \alpha_i \vec{y} + r_i$ be the noise prediction of layer $i$, where $\vec{y}$ is the unit ground-truth direction, $\alpha_i$ represents signal strength, and $r_i$ is a zero-mean residual vector. Let $\mathcal{K}$ be the set of $k$ layers selected by CoG.
The normalized CoG estimator
\[
    \widehat{y}_k = \frac{\sum_{i \in \mathcal{K}} d_i}{\left\| \sum_{i \in \mathcal{K}} d_i \right\|}
\]
converges to the true direction $\vec{y}$ as the number of selected layers $k \to \infty$, provided that the expected signal strength of selected layers is positive, i.e. $\mathbb{E}[\alpha_i \mid i \in \mathcal{K}] = \mu > 0$.
\end{lemma}

\begin{proof}
Let $s_k = \frac{1}{k} \sum_{i \in \mathcal{K}} d_i$ be the mean vector of the selected layers. Substituting the decomposition $d_i = \alpha_i \vec{y} + r_i$ into the sum, we have:
\[
    s_k = \left( \frac{1}{k} \sum_{i \in \mathcal{K}} \alpha_i \right) \vec{y} + \left( \frac{1}{k} \sum_{i \in \mathcal{K}} r_i \right).
\]
We analyze the asymptotic behavior of the two terms on the right-hand side as $k \to \infty$ by applying the Law of Large Numbers (LLN):

\begin{enumerate}
    \item \textbf{Signal Term:} The coefficient of $\vec{y}$ is the sample mean of the signal strengths $\alpha_i$. By the LLN, this converges to the conditional expectation of $\alpha$:
    \[
        \frac{1}{k} \sum_{i \in \mathcal{K}} \alpha_i \xrightarrow{k \rightarrow \infty} \mu.
    \]
    \item \textbf{Residual Term:} The second term is the sample mean of the independent residual vectors $r_i$. Since $\mathbb{E}[r_i] = \mathbf{0}$, the LLN dictates:
    \[
        \frac{1}{k} \sum_{i \in \mathcal{K}} r_i \xrightarrow{k \rightarrow \infty} \mathbf{0}.
    \]
\end{enumerate}

Combining these results, the unnormalized estimator converges to the scaled ground truth:
\[
    s_k \xrightarrow{k \rightarrow \infty} \mu \vec{y}.
\]
Since $\mu > 0$, the magnitude $\|s_k\|$ converges to $\mu \|\vec{y}\| = \mu$. Finally, the normalized estimator $\widehat{y}_k$ satisfies:
\[
    \widehat{y}_k = \frac{s_k}{\|s_k\|} \xrightarrow{k \rightarrow \infty} \frac{\mu \vec{y}}{\mu} = \vec{y}.
\]
\end{proof}

\section{Optimizing Arbitrary Concepts with VLMs.}
\label{sec:vlm}
While we are bound to concepts for which valid metrics exist when reporting results in the experiment section, we find that using vision-language models (VLMs) as judges for layer-finding makes Concept Guidance available for arbitrary concepts, even when no valid metrics exist. We demonstrate the broad applicability of Concept Guidance by optimizing for the following additional targets (\Cref{fig:vlm_targets}):

\begin{itemize}
    \item \textbf{Symmetry and Geometric Regularity:} We optimize for images whose main subject displays strong structural symmetry and an overall regular geometric arrangement. This property is relevant across portraits, architecture, and design, yet it is difficult to measure.

    \item \textbf{Background Separation:} We target images where the subject is clearly separated and visually dominant. Subject-background separation is a key component in product photography, close-up shots, and portraits.

    \item \textbf{Ukiyo-e Style:} We optimize for images that match the visual characteristics of traditional Japanese ukiyo-e woodblock prints. This shows that Concept Guidance can target highly specific artistic styles.
\end{itemize}

\begin{figure}[H]
    \centering
    \includegraphics[width=\linewidth]{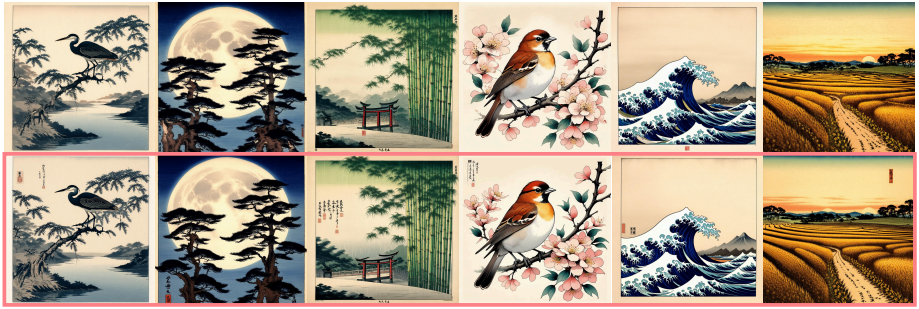}
    \caption{Comparison of Classifier-Free Guidance \cite{ho2022classifier_free_guidance} and Concept Guidance for additional targets where layers are profiled using a VLM (\texttt{InternVL3-14B}).}    
    \label{fig:vlm_targets}
\end{figure}

\section{Additional Ablation Results}
\label{sec:ablation_appendix}
This section provides the detailed ablation studies summarized in \Cref{sec:ablations}.

\paragraph{Leveraging Concept-Wise Layers} To measure the effect of leveraging different layers for guidance per concept, we compare CoG against STG \cite{hyung2025spatiotemporal}, where a fixed single layer is skipped based on overall generation quality, regardless of the given concept. While STG is a method for video generation guidance, we port the approach to text-to-image generation by choosing the best overall layer based on FID scores. To demonstrate the effect achieved by per-concept information alone, we do not leverage performance-weighted multi-skipping, i.e.\ we use CoG$_{single}$ only. We find that incorporating per-layer, per-concept knowledge when choosing the skipped layer yields significant performance increases of up to $24.5\%$, and an average increase of $5.6\%$ (\Cref{tab:stg_vs_cg}).

\begin{table}[t]
\centering
\caption{\textbf{Leveraging per-concept layer information yields significant gains} over guiding with a fixed layer as is the case in Spatio-Temporal Skip-Guidance.}
\setlength{\tabcolsep}{10pt} % Adjust this value to make columns wider or narrower
\begin{tabular}{l l *{3}{c}}
\toprule
\textbf{Model} & \textbf{Method} & \textbf{Text} (↑) & \textbf{Hands} (↑) & \textbf{Aesthetics} (↑) \\
\midrule

\multirow{2}{*}{PixArt-$\alpha$}
& STG \cite{hyung2025spatiotemporal}& - & 0.384 & 6.502 \\
& CoG$_{single}$ & - & \textbf{0.478}\rlap{ \perc{24.48\%}} & \textbf{6.606}\rlap{ \perc{1.6\%}} \\
\midrule

\multirow{2}{*}{SD3}
& STG \cite{hyung2025spatiotemporal}& 0.441 & 0.662 & 6.359 \\
& CoG$_{single}$ & \textbf{0.462}\rlap{ \perc{4.8\%}} & \textbf{0.664}\rlap{ \perc{0.3\%}} & \textbf{6.359}\rlap{ \perc{0.0\%}} \\
\midrule

\multirow{2}{*}{SD3.5}
& STG \cite{hyung2025spatiotemporal}& 0.531 & 0.540 & 6.297 \\
& CoG$_{single}$ & \textbf{0.564}\rlap{ \perc{6.2\%}} & \textbf{0.603}\rlap{ \perc{11.7\%}} & \textbf{6.344}\rlap{ \perc{0.7\%}} \\
\midrule

\multirow{2}{*}{FLUX.1-dev}
& STG \cite{hyung2025spatiotemporal} & 0.490 & 0.595 & 6.724 \\
& CoG$_{single}$ & \textbf{0.500}\rlap{ \perc{2.0\%}} & \textbf{0.643}\rlap{ \perc{8.1\%}} & \textbf{6.865}\rlap{ \perc{2.1\%}} \\
\bottomrule
\end{tabular}

\label{tab:stg_vs_cg}
\end{table}

\paragraph{Number of Skipped Layers} The idea behind CoG is to approximate the direction of highest target performance $\vec{y}$ more accurately for a larger $k$, which we also demonstrate theoretically in our convergence analysis. However, we also expect that in practice, there is a tradeoff between increased accuracy and negative synergies between multiple noise predictions for a larger $k$. In particular, we find that while multiple layers can each be beneficial, their resulting predictions may diverge. Thus, skipping too many layers with CoG can result in guiding with possibly conflicting noise predictions. To find the optimal tradeoff, we evaluate models with up to 5 skipped layers, and measure their impact on task-specific metrics. As shown in \Cref{fig:numlayers_ablation} and \Cref{tab:num_layers_ablation}, performance indeed generally improves with more skipped layers up to a certain point, with the optimal tradeoff being usually achieved around 2--3 skipped layers. Notably, significant gains are already achieved by skipping a single, concept-relevant layer.

\begin{table}[h]
	\centering
	\setlength{\tabcolsep}{3.5pt}
	\renewcommand{\arraystretch}{1.02}

	\begin{minipage}[b]{0.465\linewidth}
		\centering
        \caption{\textbf{Performance-weighted aggregation is key for Concept Guidance.}
		We compare skipping all layers in the same pass (\textit{Naive}) equal weighting of multiple predictions (\textit{Uniform}), and performance-based weighting (CoG).}
		\label{tab:guidance_methods}

		\resizebox{\linewidth}{!}{%
			\begin{tabular}{llccc}
				\toprule
				\multirow{2}{*}{\textbf{Model}} & \multirow{2}{*}{\textbf{Metric}} &
				\multicolumn{3}{c}{\textbf{Guidance Method}} \\
				\cmidrule(lr){3-5}
				  & & \textbf{Naive} & \textbf{Uniform} & \textbf{CoG (ours)} \\
				\midrule
				\multirow{2}{*}{PixArt-$\alpha$}
				  & Hands ($\uparrow$)      & 0.481 & 0.482 & \textbf{0.509} \\
				  & Aesthetics ($\uparrow$) & 6.470 & 6.714 & \textbf{6.722} \\
				\midrule
				\multirow{3}{*}{SD3}
				  & Text ($\uparrow$)       & 0.453 & 0.477 & \textbf{0.490} \\
				  & Hands ($\uparrow$)      & 0.571 & 0.669 & \textbf{0.684} \\
				  & Aesthetics ($\uparrow$) & 6.009 & 6.338 & \textbf{6.348} \\
				\midrule
				\multirow{3}{*}{SD3.5}
				  & Text ($\uparrow$)       & 0.575 & 0.591 & \textbf{0.610} \\
				  & Hands ($\uparrow$)      & 0.664 & 0.677 & \textbf{0.679} \\
				  & Aesthetics ($\uparrow$) & 6.321 & \textbf{6.368} & 6.361 \\
				\midrule
				\multirow{3}{*}{FLUX.1-dev}
				  & Text ($\uparrow$)       & 0.483 & 0.478 & \textbf{0.510} \\
				  & Hands ($\uparrow$)      & 0.159 & 0.666 & \textbf{0.675} \\
				  & Aesthetics ($\uparrow$) & 6.628 & 6.607 & \textbf{6.657} \\
				\bottomrule
			\end{tabular}
		}
	\end{minipage}\hfill
	\begin{minipage}[b]{0.515\linewidth}
		\centering
        \caption{\textbf{Increasing $k$ trades off gains vs.\ interference, with a clear sweet spot.}
		We sweep $k$ (top-$k$ improving layers); settings with fewer than five improving layers omit larger-$k$ results. Across architectures, the optimum is approximately $k=3$.}
		\label{tab:num_layers_ablation}

		\resizebox{\linewidth}{!}{%
			\begin{tabular}{llccccc}
				\toprule
				\multirow{2}{*}{\textbf{Model}} & \multirow{2}{*}{\textbf{Metric}} &
				\multicolumn{5}{c}{\textbf{Number of Layers}} \\
				\cmidrule(lr){3-7}
				  & & \textbf{1} & \textbf{2} & \textbf{3} & \textbf{4} & \textbf{5} \\
				\midrule
				\multirow{2}{*}{PixArt-$\alpha$$^\dagger$}
				  & Hands ($\uparrow$)      & 0.473 & 0.484 & 0.502 & \textbf{0.535} & 0.501 \\
				  & Aesthetics ($\uparrow$) & 6.680 & 6.709 & \textbf{6.729} & - & - \\
				\midrule
				\multirow{3}{*}{SD3}
				  & Text ($\uparrow$)       & 0.428 & 0.445 & \textbf{0.455} & 0.444 & 0.452 \\
				  & Hands ($\uparrow$)      & 0.632 & 0.666 & 0.653 & 0.675 & \textbf{0.704} \\
				  & Aesthetics ($\uparrow$) & 6.194 & 6.257 & \textbf{6.259} & 6.191 & 6.169 \\
				\midrule
				\multirow{3}{*}{SD3.5}
				  & Text ($\uparrow$)       & 0.536 & \textbf{0.582} & 0.554 & 0.539 & 0.548 \\
				  & Hands ($\uparrow$)      & 0.627 & \textbf{0.692} & 0.684 & 0.649 & 0.665 \\
				  & Aesthetics ($\uparrow$) & \textbf{6.376} & 6.295 & 6.309 & 6.291 & 6.295 \\
				\midrule
				\multirow{3}{*}{FLUX.1-dev$^\dagger$}
				  & Text ($\uparrow$)       & 0.481 & 0.478 & \textbf{0.509} & 0.491 & 0.487 \\
				  & Hands ($\uparrow$)      & 0.657 & 0.667 & \textbf{0.670} & 0.655 & 0.641 \\
				  & Aesthetics ($\uparrow$) & 6.652 & 6.654 & \textbf{6.662} & - & - \\
				\bottomrule
			\end{tabular}
		}
	\end{minipage}
\end{table}

\paragraph{Multiple Layers and Contribution Weighting} To demonstrate the benefit of computing weighted averages of per-layer noise predictions, we compare our method against two additional baselines. For \textit{Uniform} skip-guidance, separate noise predictions are computed per layer but are not weighted based on performance. In \textit{Naive} skip-guidance, multiple layers are skipped within the same forward pass. Notably, \textit{Naive} corresponds to the current implementation in the HuggingFace SD3 pipeline~\cite{huggingfaceSD3}. We hypothesize that \textit{Naive} leads to degraded images due to strong manifold distortions induced by skipping several layers simultaneously. Moreover, we expect that without weighting, models underperform due to the lack of directional control on the noise manifold. Our experiments provide quantitative (\Cref{tab:guidance_methods}) and qualitative (\Cref{fig:multiskip_ablation}) evidence that separate noise predictions are necessary and that contribution based weighting improves performance. Notably, our findings indicate that popular diffusion pipelines could be significantly improved by incorporating our method of computing separate predictions and combining them through weighted aggregation.

\begin{figure}[h]
    \centering
    \begin{minipage}[t]{0.48\linewidth}
        \centering
        \includegraphics[width=\linewidth]{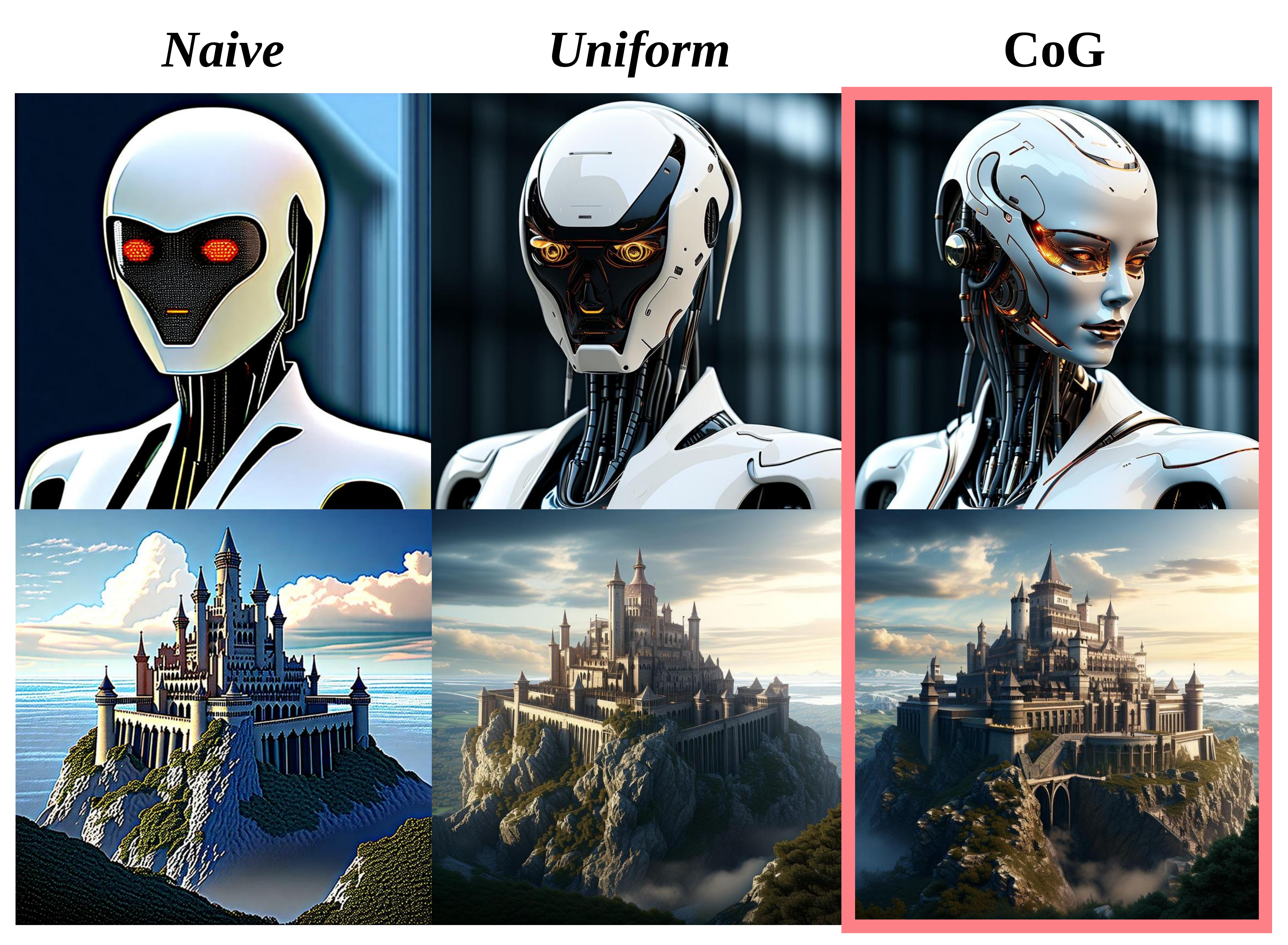}
        \caption{\textbf{Guidance Method Ablation.}}
        \label{fig:multiskip_ablation}
    \end{minipage}
    \begin{minipage}[t]{0.48\linewidth}
        \centering
        \includegraphics[width=\linewidth]{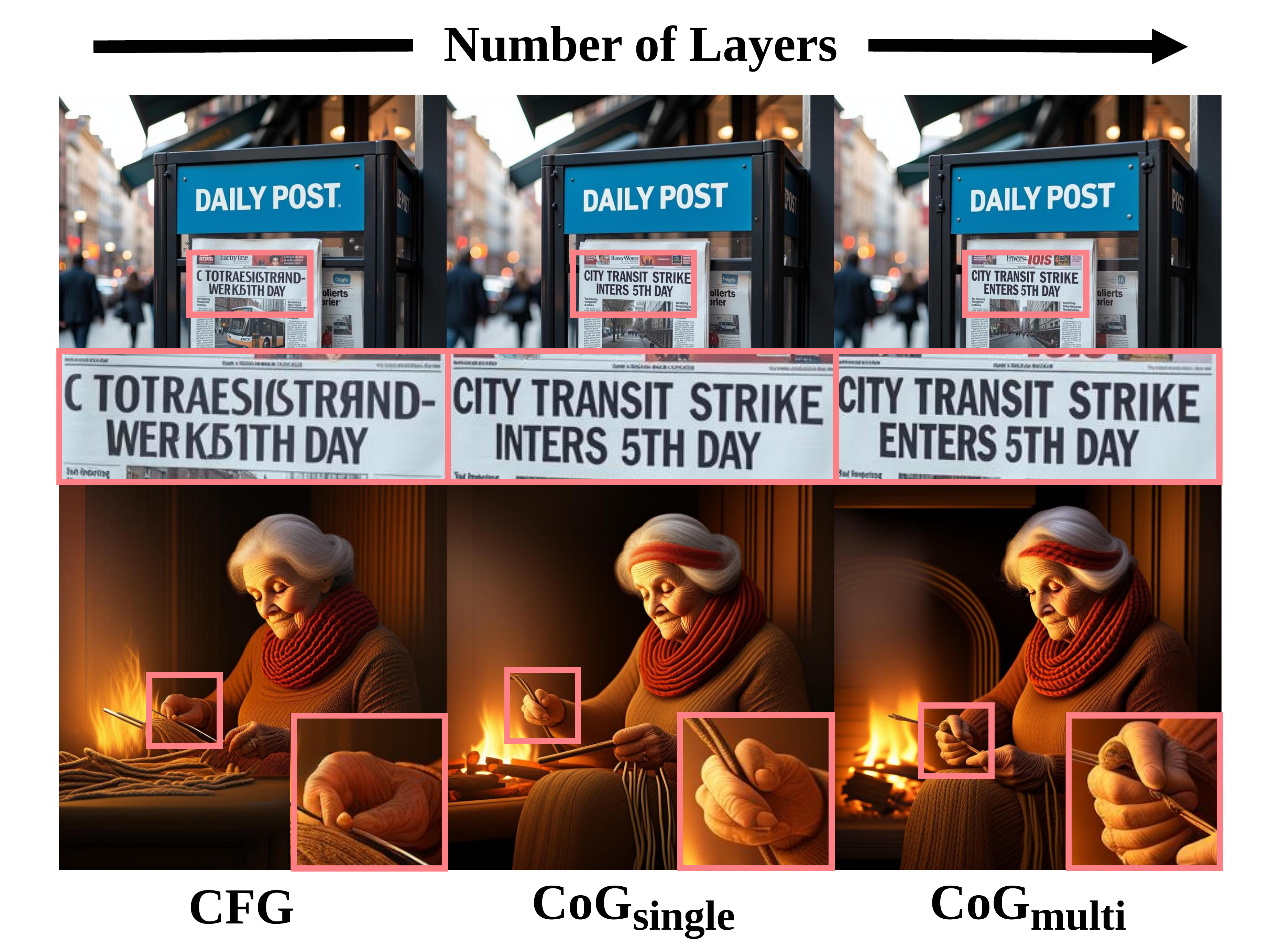}
        \caption{\textbf{Number of Layers Ablation.}}
        \label{fig:numlayers_ablation}
    \end{minipage}\hfill
    
\end{figure}

\section{Implementation Details}
\label{sec:impl_details}

\subsection{Code Availability}
Our implementation of Concept Guidance, together with the layer profiling code and the per-model, per-concept layer configurations reported in \Cref{tab:combined_layers}, is publicly available at \url{https://github.com/CompVis/concept\_guidance}.

\subsection{Computational Resources} For all experiments, models, and tasks, we use nodes of four NVIDIA A100 GPUs with 80 GBs of VRAM.

\subsection{Inference and Profiling Cost}
\label{sec:cost}
Concept Guidance introduces a one-time, offline profiling stage and a small online inference overhead. For profiling, we use $N=100$ prompts per concept. For example, for SD3 with $L=24$ layers this amounts to $N \times L = 2400$ generations and completes in $\sim 4$ hours on a single A100. Profiling is performed only once per model/concept pair (\Cref{alg:msg_profiling,alg:msg_inference}), and we release the resulting layer configurations (\Cref{tab:combined_layers}) so that CoG can be applied directly, without any profiling, for the models and concepts we study.

At inference, the cost scales linearly with the number of skipped layers $k$, since each skipped layer requires one additional noise prediction (\Cref{tab:compute_time}). The additional cost per skipped layer ranges from $0.4$\,s for PixArt-$\alpha$ to $12.4$\,s for FLUX.1-dev, i.e.\ it grows with model size; relative to a single CFG pass this corresponds to roughly $10\%$ for PixArt-$\alpha$ and up to nearly a full additional pass for FLUX.1-dev. As shown in \Cref{tab:main_results}, significant gains are already obtained with a single skipped layer ($\text{CoG}_{single}$), so the overhead is opt-in and can be tuned to the desired level of concept control.

\begin{table}[b]
    \centering
    \caption{\textbf{Inference time (s) on an A100 GPU.}}
   \begin{tabular}{l *{4}{>{\centering\arraybackslash}p{2.2 cm}}}
    \toprule
        \textbf{Model} & \textbf{CFG} \cite{ho2022classifier_free_guidance} & \textbf{CoG}$_{\mathbf{single}}$ (\textit{k}=1) & \textbf{CoG}$_{\mathbf{multi}}$ (\textit{k}=2) & \textbf{+Sec / Layer} \\
        \midrule
        PixArt-$\alpha$      & 3.5 & 3.9 & 4.3 & 0.4 \\
        Stable Diffusion 3   & 4.4 & 6.2 & 8   & 1.8 \\
        Stable Diffusion 3.5 & 5.9 & 7.7 & 9.5 & 1.8 \\
        FLUX.1-dev           & 12.9 & 25.3 & 37.7 & 12.4 \\
        \bottomrule
    \end{tabular}
    \label{tab:compute_time}
\end{table}

\subsection{Guidance Scales}
\paragraph{Classifier-Free Guidance} For CFG \cite{ho2022classifier_free_guidance}, we use the standard guidance scale per model. For Flux.1-dev \cite{huggingfaceFLUX}, the default guidance scale is $3.5$, for PixArt-$\alpha$ \cite{chen2023pixart}, it is $4.5$, and for both Stable Diffusion 3 and 3.5 \cite{huggingfaceSD3,huggingfaceSD35large}, the default scale is $7.0$.

\paragraph{Concept Guidance} To measure individual layer performance with residual skipping, we use a guidance scale of $2.0$, as used in single-skip residual STG \cite{hyung2025spatiotemporal}. For CoG, we sweep over guidance scales ranging from $1.25$ to $3.00$ in $0.25$ steps, and find that a guidance scale between $2.0$ and $2.5$ generally works best to achieve the results reported in \Cref{tab:main_results}. Full results are reported in \Cref{tab:scale_ablation}.

\begin{table}[h]
\caption{\textbf{CoG Guidance Scale Sweep.} For most models and tasks, the optimal guidance scale falls in the range between $2.0$ and $2.5$. $^{\dagger}$PixArt-$\alpha$ does not generate any visible text.}
\centering
\setlength{\aboverulesep}{0pt}
\setlength{\belowrulesep}{0pt}
\renewcommand{\arraystretch}{1.25}

\begin{tabular}{llccc>{\columncolor{blue!5}}c>{\columncolor{blue!5}}c>{\columncolor{blue!5}}ccc}
\toprule
\multirow{2}{*}{\textbf{Model}} & \multirow{2}{*}{\textbf{Metric}} &
\multicolumn{8}{c}{\textbf{Guidance Scale ($\lambda$)}} \\
\cmidrule(lr){3-10}
& & \textbf{1.25} & \textbf{1.50} & \textbf{1.75} & \textbf{2.00} & \textbf{2.25} & \textbf{2.50} & \textbf{2.75} & \textbf{3.00} \\
\midrule
\multirow{3}{*}{PixArt-$\alpha$$^{\dagger}$}
& Text ($\uparrow$)        & - & - & - & - & - & - & - & - \\
& Hands ($\uparrow$)       & 0.396 & 0.373 & 0.418 & \textbf{0.504} & 0.503 & 0.496 & 0.492 & 0.468 \\
& Aesthetics ($\uparrow$)  & 6.647 & 6.678 & 6.695 & 6.723 & 6.716 & \textbf{6.759} & 6.740 & 6.734 \\
\midrule
\multirow{3}{*}{SD3}
& Text ($\uparrow$)        & 0.482 & 0.472 & 0.486 & 0.489 & \textbf{0.506} & 0.481 & 0.480 & 0.474 \\
& Hands ($\uparrow$)       & 0.602 & 0.631 & 0.660 & \textbf{0.670} & 0.640 & 0.662 & 0.605 & 0.629 \\
& Aesthetics ($\uparrow$)  & 6.169 & 6.198 & 6.250 & 6.248 & 6.257 & \textbf{6.260} & 6.217 & 6.241 \\
\midrule
\multirow{3}{*}{SD3.5}
& Text ($\uparrow$)        & 0.561 & 0.569 & 0.581 & 0.600 & 0.595 & \textbf{0.606} & 0.602 & 0.598 \\
& Hands ($\uparrow$)       & 0.603 & 0.628 & 0.636 & 0.634 & \textbf{0.653} & 0.633 & 0.645 & 0.628 \\
& Aesthetics ($\uparrow$)  & 6.324 & 6.337 & 6.349 & \textbf{6.419} & 6.388 & 6.378 & 6.351 & 6.356 \\
\midrule
\multirow{3}{*}{FLUX.1-dev}
& Text ($\uparrow$)        & 0.502 & 0.509 & 0.521 & 0.515 & \textbf{0.524} & 0.521 & 0.515 & 0.503 \\
& Hands ($\uparrow$)       & 0.642 & 0.666 & \textbf{0.695} & 0.617 & 0.617 & 0.620 & 0.635 & 0.638 \\
& Aesthetics ($\uparrow$)  & 6.663 & 6.711 & 6.766 & 6.753 & 6.789 & 6.779 & 6.785 & \textbf{6.815} \\
\bottomrule
\end{tabular}

\label{tab:scale_ablation}
\end{table}

\subsection{Conditional Prompts}
To ensure optimal diversity of conditional prompts, we generate prompts with different large language models. Specifically, we encourage LLMs to cover a wide range of prompts with respect to length, complexity and general theme. For instance, in text generation, we make sure that prompts contain a varying count of text instances, and that text within these instances has varying complexity. For text specifically, we find that strong models like FLUX.1-dev seldom struggle with prompts containing very simple text, thus we slightly adjust the complexity to model strength for evaluation. We provide sample prompts from our datasets for text (\Cref{fig:textprompts}), hands (\Cref{fig:handsprompts}) and aesthetics (\Cref{fig:aestheticsprompts}).
\newpage
\begin{figure}[H]
    \centering
    \includegraphics[width=0.63\linewidth]{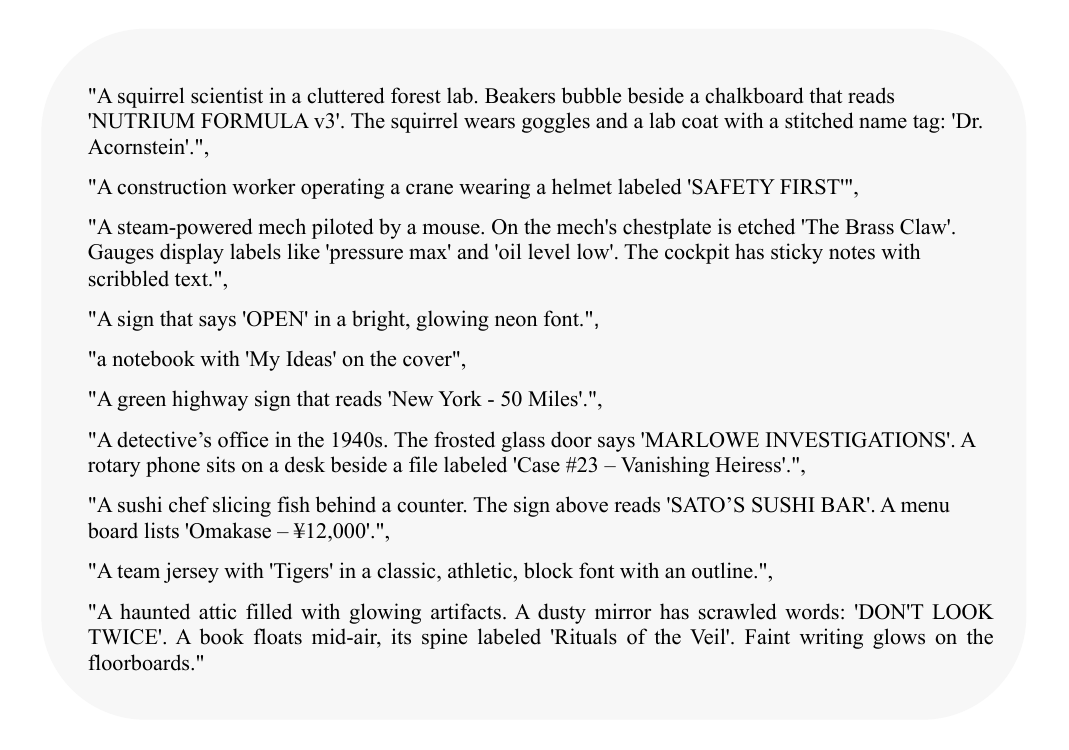}
    \caption{\textbf{Prompt Samples for Text Generation.}}
    \label{fig:textprompts}
\end{figure}

\begin{figure}[H]
    \centering
    \includegraphics[width=0.63\linewidth]{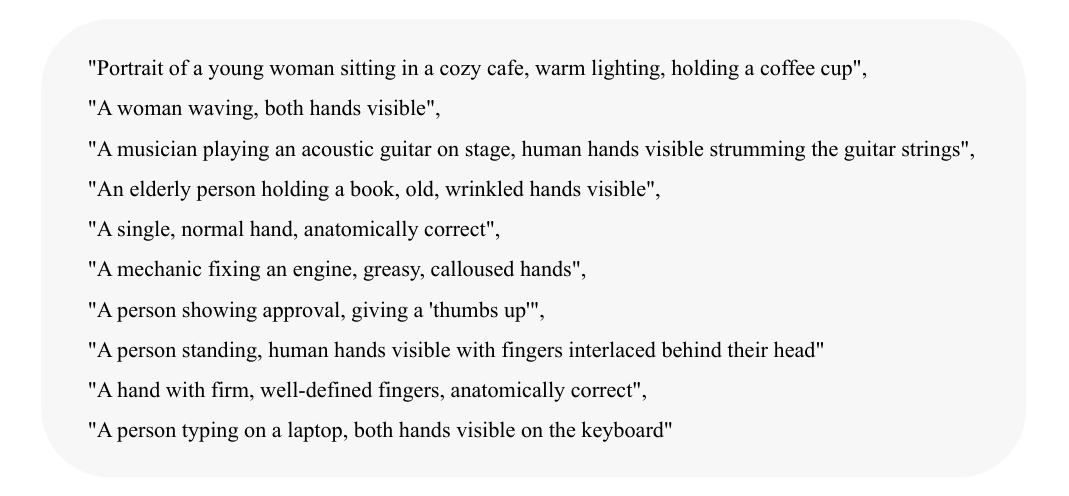}
    \caption{\textbf{Prompt Samples for Hand Generation.}}
    \label{fig:handsprompts}
\end{figure}

\begin{figure}[H]
    \centering
    \includegraphics[width=0.63\linewidth]{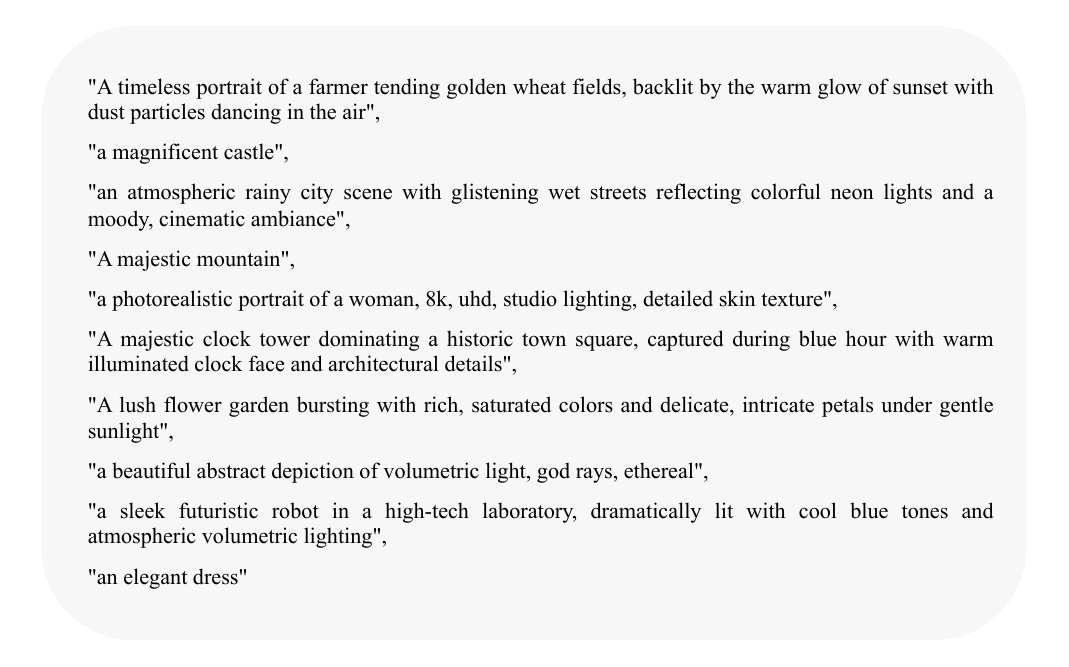}
    \caption{\textbf{Prompt Samples for Aesthetics.}}
    \label{fig:aestheticsprompts}
\end{figure}
\newpage

\subsection{Algorithms}

We implement CoG according to \Cref{alg:msg_inference} and layer profiling according to \Cref{alg:msg_profiling}. Notably, once layers have been found using layer profiling, CoG can be added into existing pipelines by exchanging the layer-wise forward with a conditional forward that returns the input if for the current layer $i$, $i \in K$, according to \Cref{eq:skip}, and implementing guidance with the weighted negative noise prediction according to \Cref{eq:msg,eq:msg_guidance}.

\noindent
\begin{center}
        \begin{minipage}{\linewidth}
            \begin{algorithm}[H]
                \caption{Concept Guidance (CoG) Step}
                \label{alg:msg_inference}
                \begin{algorithmic}[1]
                \Require $x_t, t, c, \mathcal{K}, \Omega, \text{CoG Scale } \lambda$, CFG Scale $\gamma$
                \State $\epsilon_{unc}, \epsilon_{cond} \gets \epsilon_\theta(x_t, t, \emptyset), \epsilon_\theta(x_t, t, c)$
                \State $\epsilon_{\theta} \gets \epsilon_{unc} + \gamma (\epsilon_{cond} - \epsilon_{unc})$ \Comment{Standard CFG}
    
                \State $\epsilon_{sum} \gets 0, W \gets 0$ 
                \For{$i \in \mathcal{K}$} \Comment{Aggregate Skipped Predictions}
                    \State $\epsilon_{skip} \gets \epsilon_{[\theta\setminus i]}(x_t, t, c)$
                    \State $\epsilon_{sum} \gets \epsilon_{sum} + \omega_i \cdot \epsilon_{skip}$
                    \State $W \gets W + \omega_i$
                \EndFor
                \State $\epsilon_{neg} \gets \epsilon_{sum} / W$
    
                \State \Return $(1 - \lambda)\epsilon_{neg} + \lambda \epsilon_{\theta}$ \Comment{Apply CoG}
                \end{algorithmic}
            \end{algorithm}
        \end{minipage}%
    
\end{center}

\noindent
\begin{center}
        \begin{minipage}{\linewidth} % Render internally at full width
            \begin{algorithm}[H]
                \caption{CoG Layer Profiling}
                \label{alg:msg_profiling}
                \begin{algorithmic}[1]
                \Require Model $\theta$, Prompts $\mathcal{C}$, Metric $\Phi(\cdot)$, Top-$k$ count
                \For{$c \in \mathcal{C}$} \Comment{Compute Baseline Performance}
                    \State $x_0 \gets \text{Sample}(\theta, c)$ 
                    \State $p_{\emptyset} \gets \Phi(x_0)$
                \EndFor
                \State $p_{\emptyset} \gets p_{\emptyset} / |\mathcal{C}|$
    
                \State $\mathcal{L} \gets [\text{ }]$ \Comment{Initialize Impact List}
                \For{layer $i \in \{1, \dots, L\}$} \Comment{Compute Layer Impact}
                    \For{$c \in \mathcal{C}$}
                        \State $x_0 \gets \text{Sample}(\theta_{\setminus i}, c)$ 
                        \State $p_i \gets \Phi(x_0)$
                    \EndFor
                    \State $p_i \gets p_i / |\mathcal{C}|$
                    \State $\omega_i \gets \max(0, p_i - p_{\emptyset})$ 
                    \State Append $(i, \omega_i)$ to $\mathcal{L}$
                \EndFor
    
                \State Sort $\mathcal{L}$ by $\omega_i$ descending \Comment{Select Top Layers}
                \State $\mathcal{K} \gets \{i \mid (i, \omega_i) \in \mathcal{L}[:k]\}$
                \State $\Omega \gets \{\omega_i \mid i \in \mathcal{K}\}$
                \State \Return $\mathcal{K}, \Omega$
                \end{algorithmic}
            \end{algorithm}
        \end{minipage}%
\end{center}

\section{Uncurated Samples}
\label{sec:uncurated}
Complementing the qualitative results in \Cref{fig:aesthetics,fig:texthands}, we show uncurated samples for the typical failure case of text generation and for the general task of generating more aesthetic images. For both tasks, we choose random seeds and generate images with Classifier-Free Guidance \cite{ho2022classifier_free_guidance}, Naive Skip-Guidance (as found in SD3 \cite{huggingfaceSD3}), Spatio-Temporal Skip-Guidance \cite{hyung2025spatiotemporal}, and Concept Guidance, using the same seed for each method. Uncurated samples generated with Flux.1-dev are shown in \Cref{fig:text_uncurated} for text, and in \Cref{fig:aesth_uncurated} for aesthetics.

\begin{figure}[h]
    \centering
    \includegraphics[width=0.85\linewidth]{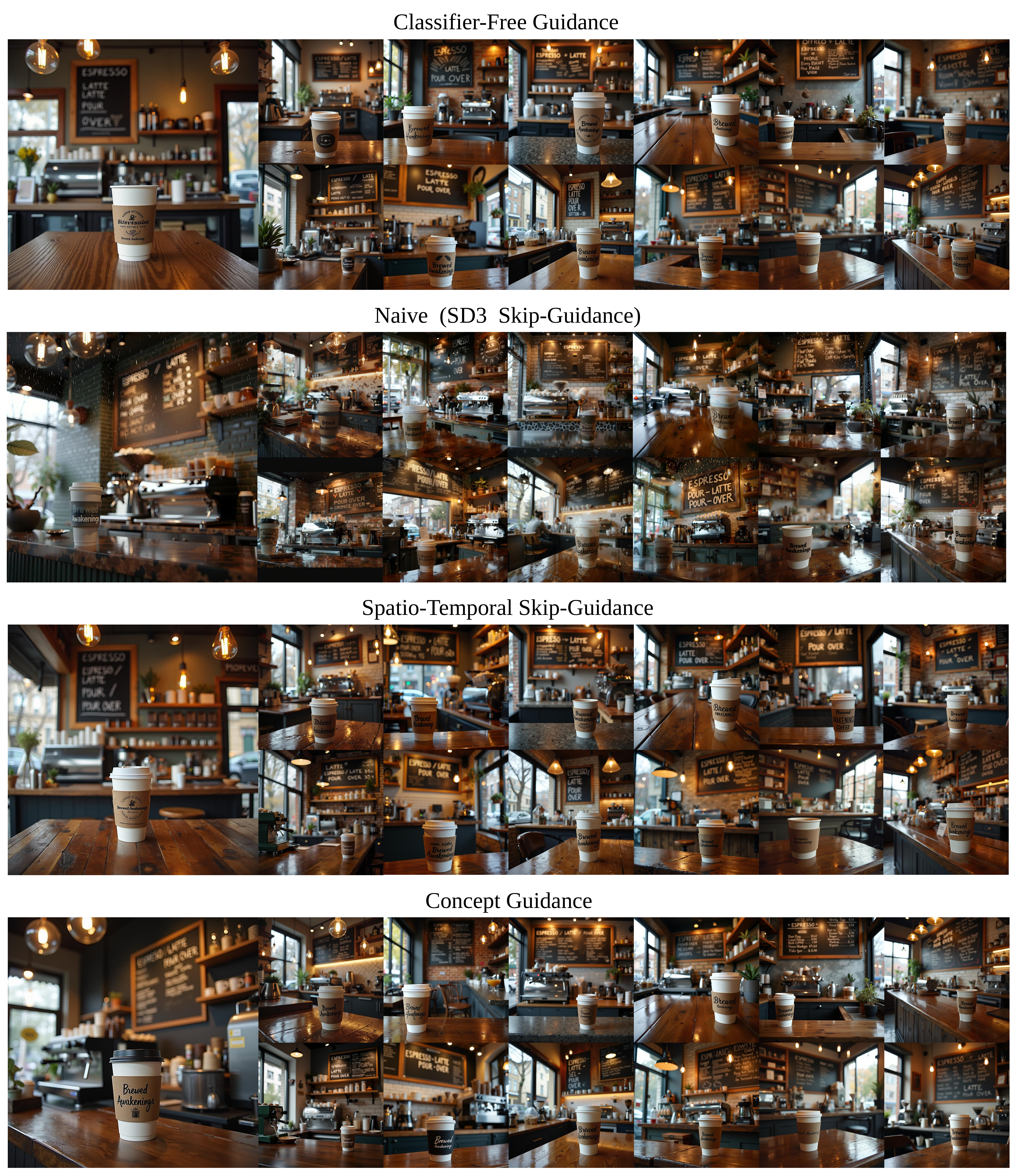}
    \caption{Uncurated Comparison of Classifier-Free Guidance (CFG) \cite{ho2022classifier_free_guidance}, Naive Skip-Guidance as found in Stable Diffusion 3 \cite{huggingfaceSD3}, Spatio-Temporal Skip Guidance \cite{hyung2025spatiotemporal} and Concept Guidance for Text Generation (FLUX.1-dev).}    
    \label{fig:text_uncurated}
\end{figure}

\begin{figure}[h]
    \centering
    \includegraphics[width=0.85\linewidth]{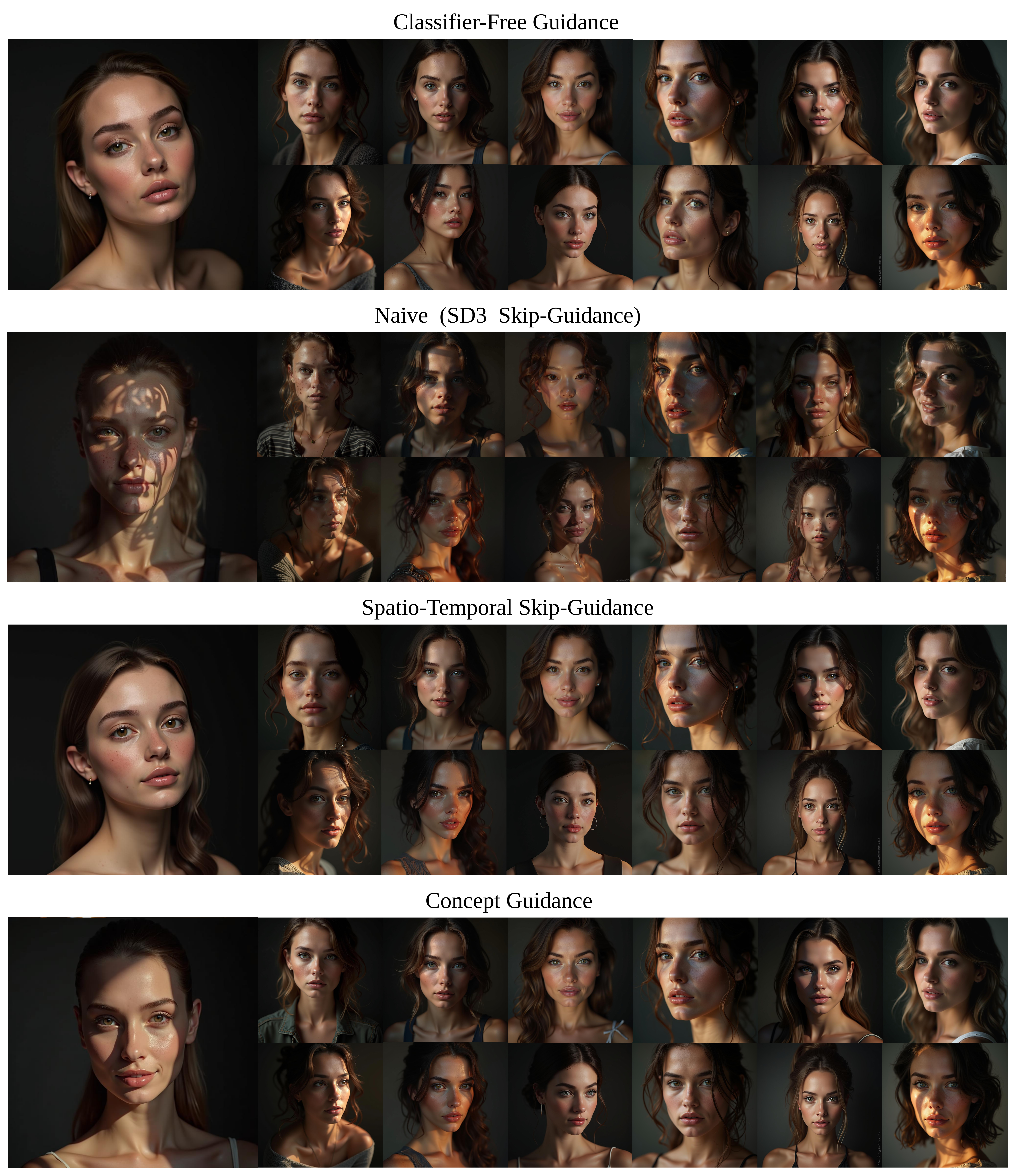}
    \caption{Uncurated Comparison of Classifier-Free Guidance (CFG) \cite{ho2022classifier_free_guidance}, Naive Skip-Guidance as found in Stable Diffusion 3 \cite{huggingfaceSD3}, Spatio-Temporal Skip Guidance \cite{hyung2025spatiotemporal} and Concept Guidance for Aesthetics (FLUX.1-dev).}    
    \label{fig:aesth_uncurated}
\end{figure}

\end{document}